\documentclass{article}
\usepackage[main, final]{neurips_2026}

\usepackage[utf8]{inputenc} 
\usepackage[T1]{fontenc}    
\usepackage{hyperref}       
\usepackage{url}            
\usepackage{booktabs}       
\usepackage{amsfonts}       
\usepackage{nicefrac}       
\usepackage{microtype}      
\usepackage{xcolor}         

\usepackage{graphicx} 
\usepackage{amsmath, amsfonts, amssymb, amsthm, dsfont}
\usepackage{algorithm}
\usepackage{algpseudocode}
\usepackage{xcolor}
\usepackage{tikz}
\usepackage{hyperref}
\usepackage{subcaption}
\usepackage{wrapfig}
\usepackage{booktabs}
\usepackage{subcaption}
\renewcommand{\arraystretch}{1.2}
\usepackage[normalem]{ulem}

\newtheorem{definition}{Definition}
\newtheorem{proposition}{Proposition}
\newtheorem{theorem}{Theorem}
\newtheorem{lemma}{Lemma}

\newcommand{\one}{\mathds{1}}
\newcommand{\indi}[1]{\mathds{1}\{#1\}}

\newcommand{\E}{\mathbb{E}}
\renewcommand{\P}{\mathbb{P}}
\newcommand{\R}{\mathbb{R}}

\newcommand{\cf}{\mathcal{F}}

\newcommand{\cl}{\mathcal{L}}

\newcommand{\cp}{\mathcal{P}}

\newcommand{\sumn}{\sum_{n=1}^N}

\newcommand{\Phat}{\hat{P}}
\newcommand{\Pbhat}{\hat{\mathbf{P}}}

\DeclareMathOperator*{\argmin}{argmin}

\title{Online Quantile Omniprediction for Proper Losses}

\author{
Hyunsuk Kim \\
Department of Statistics\\
University of California, Berkeley
\And
Isaac Gibbs \\
Department of Statistics\\
University of Toronto
\AND
Michael W. Mahoney \\
ICSI, LBNL, and Dept. of Statistics\\
University of California, Berkeley
\And
Ryan J. Tibshirani \\
Department of Statistics\\
University of California, Berkeley
}

\begin{document}

\maketitle
\begin{abstract}
The task of constructing a predictor that is simultaneously accurate under multiple evaluation metrics has been popularized recently under the name omniprediction. 
Prior work in this area has largely focused on binary prediction tasks, and existing results for multiclass and continuous outcomes typically suffer from slow error rates. 
In this paper, we consider forecasts expressed as a set of quantiles over a discrete but otherwise arbitrary set of probability levels. 
We develop sample-efficient algorithms for constructing forecasts that are accurate simultaneously across all proper losses that satisfy mild regularity conditions, where proper losses are those minimized by predicting the true quantiles. 
Our algorithms are online and can operate in a dynamic environment, while simultaneously taking multiple losses into account, and their forecasts offer the potential to provide diverse end-users with rich information for evaluating risk and informing a range of practical decision-making tasks. 
As an example, we demonstrate how our procedure can be used to ensemble quantile forecasts of COVID-19 hospitalizations made by experts during the pandemic. 
Code is available at \url{https://github.com/hysk79/online-quantile-omniprediction}. 
\end{abstract}

\section{Introduction} 
\label{sec:intro}

Probabilistic forecasting, i.e., forecasting the distribution of a target variable, has become a popular framework across a variety of prediction tasks \cite{Gneiting2014, Vannitsem2021, Schorlemmer2018, cramer2022violinplot}. 
In contrast to point predictions, probabilistic forecasts provide a richer description of uncertainty and offer end-users greater information with which to evaluate risk and guide effective decision-making \cite{Raftery2016}. 
This flexibility is particularly valuable in settings in which a single prediction is shown to a diverse set of users with varying and unknown objectives, as it allows users to extract the information most relevant to their needs.

A conventional approach for training and evaluating a probabilistic forecast is to measure its distance from the true distribution using a proper loss \cite{Gneiting2005, Gneiting2007}.  
In general, we say that a loss is proper if (in expectation) it is minimized by predicting the true distribution of the outcome. 
There are a number of popular proper losses in the literature, and a key difficulty is determining which loss one should use for training. 
Indeed, while all proper losses share the same global optimum, they penalize distance from the ground truth asymmetrically, and thus obtaining a near minimizer in one loss does not guarantee good performance for other losses~\cite{Buchweitz2025}. 

In response to this, a recent line of work aims to construct ``omnipredictors'' that achieve optimal performance on multiple losses simultaneously \cite{gopalan2022omnipredictors,  gopalan2023swap,  garg2024oracle, gopalan2024omni_regression, Noarov2025, lu2025omni_nonlinear, okoroafor2025near, isaac2025omni}. 
Much of this work has focused on binary prediction tasks. 
In that context, it has been shown that, given a baseline class of prediction functions, one can construct a single unified predictor that approximately matches the performance of the best baseline on every proper loss simultaneously  \cite{garg2024oracle,okoroafor2025near, isaac2025omni}.  
Notably, the error rates achieved in these works are identical (up to logarithmic factors) to the standard error rates for minimizing a single loss function: ($\tilde{O}(\sqrt{1/T})$ in the online setting over $T$ timesteps; and $\tilde{O}(\sqrt{1/n})$ in the offline setting over $n$ samples).  
Thus, omniprediction offers the potential to provide a practical path to improving existing learning algorithms.  

In contrast to this work on binary prediction tasks, results for multiclass and continuous targets are much more modest. 
Prior work in \cite{gopalan2022omnipredictors,gopalan2024omni_regression,lu2025omni_nonlinear} gives algorithms for multiclass or continuous outcome omniprediction over restricted classes of losses satisfying various convexity and Lipschitz conditions, while \cite{Noarov2025} considers multiclass Bregman losses that are separable in the predicted probability for each class.  
In all of these cases, omniprediction is achieved at significantly slower rates than in the binary setting. 
For more general losses, \cite{Hu2026} propose a method for obtaining omniprediction over all proper losses in the multiclass setting. 
However, its error rate deteriorates rapidly as the number of classes grows, with realized rates of $(k/n)^{1/(k+1)}$ and $(k/T)^{1/(k+1)}$ for online/offline $k$-class classification, respectively.  

To achieve more practical methods, we pursue omniprediction guarantees for simplified probabilistic forecasts expressed as dense grids of quantile predictions for the outcome.  
Our motivation for considering such quantile forecasts is two-fold. 
First, on a technical level, proper losses for quantiles admit a decomposition as mixtures over a smaller two-parameter class of elementary losses. 
This will allow us to focus our algorithms on a small, tractable set of losses and thus enable the development of efficient omniprediction algorithms that mirror those for the binary prediction setting. 
Second, quantile forecasts are a popular way to summarize uncertainty in many real-world forecasting domains \cite{Hong2016, cramer2022violinplot}. 
As a demonstration, we will show how our methods can be used to improve quantile-level forecasts of COVID-19 hospitalizations made by expert forecasting groups during the pandemic. 

The remainder of this paper is structured as follows.  
Section \ref{sec:preliminaries} formally defines our prediction framework and presents the key structural results for proper quantile losses used throughout this paper.  
Section \ref{sec:theory} develops our main method, which is an online forecasting algorithm for constructing quantile predictions that achieve low regret on all (sufficiently smooth) proper losses simultaneously. 
Analogous to results for the binary setting, we show that the omniprediction error of this method over $T$ timesteps scales as $\tilde{O}(\sqrt{1/T})$ and thus matches rates for learning under a single loss (up to logarithmic factors).  
Section \ref{sec:experiments} gives our main empirical results, where we show that our procedure can be used to improve quantile forecasts made during the COVID-19 pandemic.
We conclude in Section \ref{sec:discussion} with a discussion of limitations and directions for future work. Additional technical details, including proofs of all theoretical results, efficient subroutines for faster implementation of our algorithm, and additional empirical results and plots, are deferred to the appendix.

\section{Preliminaries} 
\label{sec:preliminaries}

\subsection{Problem statement} \label{ssec:prob_statement}

We consider an online adversarial setting in which at each timestep $t \in [T]$ the goal is to forecast an outcome $Y_t$ whose distribution is allowed to depend on the entire history of previous outcomes and predictions.  For simplicity, we assume that $Y_t$ lies in $[0,1]$, although our results readily extend to any bounded interval by translation and rescaling. 

Given quantile levels $0 \!<\! \tau_1 \!<\! \tau_2 \!<\! \cdots \!<\! \tau_N \!<\! 1$, our forecasts take the form of a vector $q(t) = (q_{1}(t),\dots,q_N(t))$ where $q_n(t)$ denotes our estimate of the $\tau_n$ quantile of $Y_t$.  We evaluate the performance of these predictions by comparing them to a finite set of $B$ base forecasters $\mathcal{F} = \{ \{f^n_b(t)\}_{n \in [N], t \in [T]} : b \in [B]\}$, where $f^n_b(t) \in [0,1]$ denotes the forecast of the $\tau_n$-quantile made by base forecaster $f_b$ at timestep $t$. 

As motivating examples, one may consider problems in which $\mathcal{F}$ is a collection of parametric functions mapping some covariates $X_t$ to $Y_t$ or problems in which $\mathcal{F}$ contains candidate forecasts from many different baseline learning algorithms. In either case, at every timestep $t$, we will observe the previous sequence of outcomes $(Y_s)_{s < t}$ and the current and preceding base forecasts $ \{ \{f^n_b(s)\}_{n \in [N], s \leq t} : b \in [B]\}$, and our goal will be to construct $q(t)$ so that it matches the performance of the best base forecaster on every loss $\ell$ in some collection $\cl$.  More precisely, following \cite{gopalan2022omnipredictors}, we aim to minimize the omniprediction error, defined as
\begin{equation}\label{eq:omni}
  \sup_{\ell \in \mathcal{L}, b \in [B]} \frac{1}{T}\sum_{t=1}^T \ell(q(t),Y_t) - \ell(f_b(t),Y_t),
\end{equation}
where $\cf = \{f_b\}_{b=1}^B$ and $f_b(t):=(f_b^1(t),...,f_b^N(t))$. It is important to emphasize that in this definition the comparator forecaster in $\mathcal{F}$ is allowed to depend on the loss $\ell \in \mathcal{L}$. As a result, to achieve low omniprediction error it is not sufficient to set $q(t)$ to be a fixed element of $\mathcal{F}$, as the best base forecaster may vary between losses.

We will be interested in pursuing low omniprediction error with respect to the class of proper losses. Broadly, we say that a loss function is proper for forecasting a given set of quantiles if it is minimized by forecasting the true quantiles of the target.\footnote{Losses with this property are also sometimes referred to as consistent losses or consistent scoring functions.} This is formalized in the following definition.
\begin{definition}[Proper loss for quantile levels]
For a distribution $P^y$ on $[0,1]$, denote its $\tau$-level quantile by $q^*_{\tau}(P^y) := \inf\{q \in [0,1] : \mathbb{P}_{Y \sim P^y}(Y \geq q) \geq \tau\}$. A loss function $\ell : [0,1]^N \times [0,1] \to \mathbb{R}_{\geq 0}$ is referred to as a \textbf{proper loss for quantile levels}  $0 < \tau_1  < \tau_2 < \cdots < \tau_N < 1 $ if 
\[
(q^*_{\tau_n}(P^y))_{n=1}^N \in \argmin_{q \in \mathbb{R}^N} \mathbb{E}_{Y \sim P^y}[\ell(q,Y)], \quad \text{ for all distributions $P^y$ supported on $[0,1]$}.
\]
\end{definition}
By definition, the class of proper losses is the largest class of losses that share the true quantiles as their common minimizer, and thus it is the largest class over which we can reasonably expect to achieve low omniprediction error.

\subsection{Existing characterizations of proper losses}\label{ssec:characterization}

In order to make our omniprediction objective tractable, we will rely on structural characterizations of  proper losses for quantiles. Such characterizations have a long history in the forecasting literature, dating back to the pioneering work of \cite{savage1971decompose}.  We will in particular rely on results for quantiles developed in \cite{thomson1979decompose, Saerens2000, gneiting2011point, ehm2016binarydecomp} which show that any (sufficiently smooth) proper loss for a single quantile level can be decomposed as a mixture over a one-dimensional family of elementary losses. To outline this result, \cite{ehm2016binarydecomp} summarizes the characterization of proper loss for single quantile level from previous works as follows:

\paragraph{Characterization.}
For any non-decreasing function $g: [0,1] \to \R$, the loss function $\ell : [0,1] \times [0,1] \to \mathbb{R}_{\geq 0}$ given by
\begin{align}\label{eq:q_consistent_score}
\ell(q,y) = (\one(y<q)-\tau)(g(q)-g(y))
\end{align}
is proper for the $\tau$-level quantile. Moreover, this representation is \textbf{exhaustive}, in the sense that any proper loss function for the $\tau$-level quantile that satisfies the mild regularity conditions: (L0) $\ell(q,y) \geq 0$ with equality if $q=y$; (L1) $\ell$ is continuous; and (L2) $\ell$ is twice continuously differentiable for $q \neq y$ can be expressed in this form.

Let $\cl_{\tau}$ denote the class of loss functions of the form \eqref{eq:q_consistent_score} with $g$ being a left-continuous, non-decreasing function. Then, \cite{ehm2016binarydecomp} shows that every loss in $\cl_{\tau}$ can be decomposed as a mixture of so-called elementary losses as follows.

\begin{theorem}[Theorem 1a in \cite{ehm2016binarydecomp}] \label{thm:decomp_single_q}
Fix any $\tau \in (0,1)$. For any loss function $\ell \in \mathcal{L}_{\tau}$, there exists a non-negative Lebesgue--Stieltjes measure $H$ on $[0,1]$, given by $dH(\theta)=dg(\theta)$, such that for all $q,y \in [0,1]$,
\begin{align}\label{eq:elem_score_q}
\ell(q,y) = \int_{0}^{1} \ell_{\tau,\theta}(q,y) dH(\theta),
\end{align}
where for any $\theta \in [0,1]$, $\ell_{\tau, \theta}$ is the elementary loss 
\begin{align}\label{eq:elementary_loss}
\ell_{\tau, \theta}(q,y) = (\one(y<q) - \tau)(\one(\theta<q) - \one(\theta<y)) 
=
\begin{cases}
  1-\tau, & y \leq \theta < q,  \\
  \tau, & q \leq \theta < y, \\
  0, & \text{otherwise}.
\end{cases}
\end{align}
\end{theorem}

It may be instructive to note that for any $\tau, \theta  \in [0,1]$ the elementary loss $\ell_{\tau,\theta}$ is itself proper.  To see this, fix any $\tau, \theta  \in (0,1)$\footnote{The cases $\tau \in \{0,1\}$ and $\theta \in \{0,1\}$ can be proven similarly and are omitted.} and note that by definition $\ell_{\tau,\theta}$ essentially only admits two forecasts: namely, we may either forecast a value $q > \theta$ or a value $q \leq \theta$. Which of these forecasts is optimal will depend on the sign of
\[
\mathbb{E}_{Y \sim P^y}[\ell_{\tau,\theta}(1,Y)] - \mathbb{E}_{Y \sim P^y}[\ell_{\tau,\theta}(0,Y)]  = (1-\tau) \mathbb{P}(y \leq \theta) - \tau \mathbb{P}(y > \theta), 
\] 
which is nonnegative (resp.~nonpositive) exactly when $\theta$ is below (resp.~above) the $\tau$-quantile of $P^y$. 

\paragraph{Decision-making interpretation of elementary losses.}

\renewcommand{\arraystretch}{1.2}
\begin{wraptable}{r}{0.29\textwidth}
\centering
\caption{Payoff table} \label{tab:payoff}
\begin{tabular}{|c|c|c|}
\hline
$s(a,o)$ & $o=1$ & $o=0$ \\
\hline
$a=1$ & $C$ & $C$ \\
\hline
$a=0$ & $L$ & $0$ \\
\hline
\end{tabular}
\end{wraptable}

Following the exposition in \cite{mills2025payoff}, one may also interpret $\ell_{\tau,\theta}$ in terms of a decision-making task in which one must decide whether to take a preventative action against an adverse event. More formally, let $C$ denote the cost of taking this action (e.g., expanding hospital capacity) and $L$ denote the price paid if the adverse event occurs (e.g., COVID-19 cases surge) and no preventative action was taken. Let $a = \one\{\text{action taken}\}$ and $o = \one\{\text{adverse event occurs}\}$ denote binary indicators for the action and outcome, respectively, and define a utility $s$ expressing the payoffs for these outcomes as in Table \ref{tab:payoff}. Finally, suppose the adverse event occurs exactly when some outcome $Y$ (e.g., COVID-19 cases) exceeds a given threshold $\theta$ (e.g., the hospital's current capacity). Then, if we take the preventative action if and only if the forecast quantile $q$ exceeds $\theta$, one can check~that 
\[
s\left(\indi{q > \theta}, \indi{Y > \theta}\right) - s(\indi{Y > \theta}, \indi{Y > \theta}) = L \cdot \ell_{1-C/L,\theta}(q,Y),
\]
and thus up to renormalization correctly forecasting the $(1-C/L)$-level quantile of $Y$ is equivalent to choosing the correct action.

To extend beyond single quantile forecasts we will use a result of \cite{fissler2016multiq} which shows that proper losses for vectors of quantiles decompose as sums of proper losses for single quantiles. 

\begin{theorem}[Simplified version of Proposition 4.2 in \cite{fissler2016multiq}] \label{thm:decomp_multi_q}
Fix any $0 < \tau_1 < \cdots < \tau_N < 1$ and let $\ell: [0,1]^N \times [0,1] \to \R_{\geq 0}$ be a proper loss for quantile levels $\{\tau_n\}_{n=1}^N$ that satisfies mild regularity conditions (See Appendix \ref{appssec:regularity}). Then, there exists losses $\ell_1,\dots,\ell_N :[0,1] \times [0,1] \to  \mathbb{R}_{\geq 0}$ such that $\ell_n$ is a proper loss for the $\tau_n$-level quantile and
\begin{align}\label{eq:ell_to_ell_n}
    \ell(q_1, \ldots, q_N, y) = \sum_{n=1}^{N} \ell_n(q_n, y).
\end{align}
\end{theorem}

Now, define the following two classes of loss functions $\ell: [0,1]^N \times [0,1] \to \R_{\geq 0}$:
\begin{align}\label{def:cl}
\cl := \left\{\ell \!=\! \sum_{n=1}^N \ell_n \,\middle|\, \ell_n(q_n,y) \!=\! \big(\one(y\!<\!q_n) \!-\! \tau_n\big) \big(g_n(q_n) \!-\! g_n(y)\big) \right\}, \,\, \cl_1 := \left\{\ell \in \cl \mid \|\ell\|_{\infty} \leq 1\right\} ,
\end{align}
where each $g_n$ is non-decreasing and left-continuous. Putting the results above together, any loss function $\ell \in \cl$ can be expressed as $\ell(q_1, \ldots, q_N, y) = \sum_{n=1}^{N} \int_{-\infty}^{\infty} \ell_{\tau_n, \theta}(q_n, y) \, dH_{n}(\theta)$ for some non-negative mixing measures $H_n$ and is a proper loss for quantile levels $\{\tau_n\}_{n=1}^N$. Moreover, this representation is \textbf{exhaustive}: $\cl$ contains \textbf{all} proper losses for quantile levels $\{\tau_n\}_{n=1}^N$ that satisfies mild regularity conditions. Finally, we add the additional restriction in $\mathcal{L}_1$ that $\ell$ is bounded because this will give a common scale to our definition of omniprediction and will ensure that the mixing distributions have bounded total measure.
\section{Main theoretical results} \label{sec:theory}

We now turn to our main theoretical results which develop an algorithm for constructing quantile forecasts that achieves low omniprediction error with respect to the loss class $\mathcal{L}_1$.

\subsection{Discretization} \label{nssec:discretization}

Our first step towards deriving our main method will be to reduce from optimization over all losses in $\mathcal{L}_1$ to optimization over a finite class. To accomplish this, we will combine the characterizations of proper losses from Section~\ref{sec:preliminaries} with a discretization lemma that reduces the number of elementary losses that we need to consider to a finite set. For simplicity, we will assume that the outcomes $Y_t$ and base forecasts $f_b^n(t)$ lie in the discrete grid $\{0, \frac{1}{m}, \ldots, 1\}$, for some large fixed value $m \in \mathbb{N}$. This assumption is relatively minor and is made primarily for technical convenience to avoid edge cases in which the elementary loss $\ell_{\tau,\theta}$ varies significantly at nearby values of $\theta$. In our experiments we will enforce these conditions explicitly by rounding the base forecasters and target outcomes to an appropriate grid.  More broadly, our results can be readily extended to other sets of assumptions. For example, if $Y_t$ is continuously distributed with a bounded density the same discretization lemma holds up to an extra additive error of size $O(1/m)$. 

Now, let $\theta_i = \frac{i}{m} - \frac{1}{2m}$ for $i \in [m]$. The following lemma shows that to control the omniprediction error over the class $\cl_1$, it suffices to control the error over the finite collection of elementary losses $\{\ell_{\tau_n,\theta_i}\}_{i \in [m], n \in [N]}$. This  discretization scheme  is visualized in Figure \ref{fig:discretization}; and a proof of this result can be found in Appendix~\ref{pf:discretize_multi}.

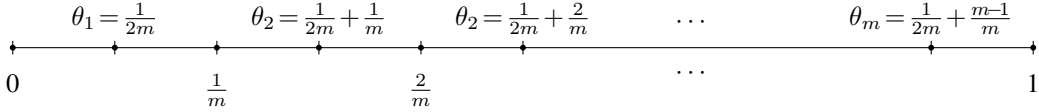
\begin{figure}[t] 
\centering
\begin{tikzpicture}[scale=0.9]
  \draw[-] (0,0) -- (15,0);
  \foreach \x/\name in {0/0, 3/$\frac{1}{m}$, 6/$\frac{2}{m}$, 15/1} {
    \draw (\x,0.1) -- (\x,-0.1) node[below=4pt] {\name};
    \filldraw[black] (\x,0) circle (1pt);
  }

  \foreach \x/\name in {1.5/$\theta_1\!=\!\frac{1}{2m}$, 4.5/$\theta_2\!=\!\frac{1}{2m}\!+\!\frac{1}{m}$, 7.5/$\theta_2\!=\!\frac{1}{2m}\!+\!\frac{2}{m}$, 13.5/$\theta_{m}\!=\!\frac{1}{2m}\!+\!\frac{m\!-\!1}{m}$} {
    \draw (\x,0.1) -- (\x,-0.1) node[above=4pt] {\name};
    \filldraw[black] (\x,0) circle (1pt);
  }
  \draw (10, 0) node[above=4pt] {$\cdots$};
  \draw (10, 0) node[below=4pt] {$\cdots$};

\end{tikzpicture}
\caption{Visualization of the discretization on the real line. Outcomes $Y_t$ and forecasts $f_b^n(t)$ take values in the grid $\{0, 1/m, ..., 1\}$ that is interleaved by the values of $\theta_i$.} \label{fig:discretization}
\end{figure}

\begin{lemma}\label{lem:discretize_multi}
Let $\cl_1$ denote the class of bounded proper losses defined in \eqref{def:cl} and suppose $Y_t, q_n(t), f^n_b(t) \in \{0, 1/m, ..., 1\}$ for all $t \in [T]$, $n \in[N]$, and $b \in [B]$. Then, 
\[
1 \leq 
\frac{\sup_{\ell \in \cl_1, b} \frac{1}{T}\sum_{t=1}^T \left[\ell(q(t),Y_t) - \ell(f_b(t),Y_t)\right]}
{\sup_{i,n,b} \frac{1}{T}\sum_{t=1}^T \left[\ell_{\tau_n, \theta_i}(q_n(t),Y_t) - \ell_{\tau_n, \theta_i}(f_b^n(t),Y_t)\right]}
\leq 2.
\]
\end{lemma}

Results analogous to Lemma \ref{lem:discretize_multi} have been used in many places in the omniprediction problem to simplify an infinite class of losses to a more tractable finite set. Most similar to our work,  \cite{isaac2025omni} gives a similar discretization result for binary outcomes $Y \in \{0,1\}$ and the class of bounded proper losses for the probability that $Y=1$. In all cases, the key conclusion is that in order to achieve low omniprediction error it sufficies to work with the discretized collection of losses. In particular, in our problem it suffices to minimize the objective
\begin{align} \label{eq:omni_objective_discretized}
\sup_{i \in [m], n \in [N], b\in[B]} \frac{1}{T}\sum_{t=1}^T \left[\ell_{\tau_n, \theta_i}(q_n(t),Y_t) - \ell_{\tau_n, \theta_i}(f_b^n(t),Y_t)\right].
\end{align}
This will be the focus of the next section.

\subsection{Algorithm}\label{nssec:algorithm}

We now introduce our main method as Algorithm~\ref{alg:omni_online}. 

To accomplish this, we will apply an algorithmic framework for multiobjective learning developed in \cite{lee2022online} and popularized in a number of recent works on omniprediction and related problems \citep[e.g.,][]{Noarov2025, garg2024oracle,lu2025omni_nonlinear,  isaac2025omni}. The core idea proposed by \cite{lee2022online} is to consider a two-player game in which one player constructs a mixture of the elementary loss functions $\ell_{\tau_n, \theta_i}$ and the other player best responds with a (randomized) predictor that minimizes the worst-case loss gap on the mixture. Following that work, in Algorithm~\ref{alg:omni_online}, this is implemented by having the first player use the well-known Hedge algorithm \citep{Vovk1990, Littlestone1994, Freund1997} to learn weights $w_{ib}^n(t)$ over the quantiles levels $\tau_n$, parameters $\theta_i$, and baselines $f^n_b$ appearing in (\ref{eq:omni_objective_discretized}), so that the time-average weighted mixture loss, 
\[
 \frac{1}{T} \sum_{t=1}^T \sum_{i,n,b} w_{ib}^n(t) \left[\ell_{\tau_n, \theta_i}(q_n(t),Y_t) - \ell_{\tau_n, \theta_i}(f_b^n(t),Y_t)\right],
\]
is guaranteed to be close to the supremum in (\ref{eq:omni_objective_discretized}). Then, letting $\Delta_m$ denote the set of distributions over $[m]$, the second player constructs their (randomized) forecast as $\mathbf{q}(t) = (q_n(t))_{n=1}^N \sim \hat{\mathbf{P}}(t)$, where $\hat{\mathbf{P}}(t) := (\hat{P}_n(t))_{n=1}^N$ solves the minmax program,
\begin{align}\label{eq:minimax}
  \hat{\mathbf{P}}(t) = \argmin_{P_1,...P_N \in \Delta_m} \max_{P^y \in \Delta_m} \sum_{i=1}^m \sumn \sum_{b=1}^B w_{ib}^n\E_{q_n\sim P_n, Y \sim P^y}[\ell_{\tau_n, \theta_i}(q_n,Y) - \ell_{\tau_n, \theta_i}(f_b^n(t),Y)].
  \end{align}
  By definition, this forecast minimizes the mixture loss over the worst-case (and potentially adversarially chosen) value of $Y_t$.

The following theorem shows that Algorithm \ref{alg:omni_online} achieves our desired omniprediction guarantee; and the proof of this result can be found in Appendix \ref{pf:omni_guarantee}.
  
\begin{theorem}\label{thm:omni_guarantee}
Suppose target values $Y_t$, forecaster predictions $f_b^n(t)$ for $t \in [T], n \in [N], b \in [B]$ each takes values in a uniform grid of size $m$. With learning rate $\eta = c\sqrt{(\log mNB)/T}$, the predictions $\Pbhat(t)=(\hat{P}_n(t))_{n=1}^N$ from Algorithm~\ref{alg:omni_online} achieve the omniprediction error bound
\begin{align*}
\sup_{\ell \in \cl_1,\, b \in [B]} \frac{1}{T} \sum_{t=1}^T \E_{q_n \sim \Phat_n(t)} \left[\ell\left((q_n)_{n=1}^N, Y_t\right) - \ell(f_b(t), Y_t)\right]
\leq 2\left(c+\frac{1}{c}\right)\left(\sqrt{\frac{\log mNB}{T}}\right).
\end{align*}
\end{theorem}

\begin{algorithm}[htbp]
  \caption{Omniprediction under all proper losses for multiple quantile levels} \label{alg:omni_online}
  \begin{algorithmic}[1]
  \State \textbf{Data:} Data stream $(Y_t)_{t=1}^T$, quantile levels $\{\tau_n\}_{n=1}^N$, grid size $m \!\in\! \mathbb{N}$, learning rate $\eta\!>\! 0$, base forecasters $\cf = \{f_b\}_{b=1}^{B}$.
  \State \textbf{Initialize:} $w_{ib}^n(1) = \frac{1}{mnB}$, for all $i \in [m], n \in [N], b \in [B]$;
  \For{$t = 1, \ldots, T$}
  \State \textbf{1) Make prediction at time $t$}:
  \begin{align}\label{eq:Pn_def_minimax}
  \Pbhat(t) = \argmin_{P_1,...P_N \in \Delta_m} \max_{P^y \in \Delta_m} \sum_{i=1}^m \sumn \sum_{b=1}^B w_{ib}^n\E_{q_n\sim P_n, Y \sim P^y}[\ell_{\tau_n, \theta_i}(q_n,Y) - \ell_{\tau_n, \theta_i}(f_b^n(t),Y)];
  \end{align}
  \State \textbf{2) Update the weights:} For all $i,n,b$,
  \begin{align}\label{eq:w_update}
  w_{ib}^n(t+1) = \frac{w_{ib}^n(t) \exp\bigl(\eta \big(\mathbb{E}_{q_n \sim \hat{P}_n(t)}[\ell_{\tau_n, \theta_i}(q_n, Y_t)] 
  - \ell_{\tau_n, \theta_i}(f_b^n(t), Y_t)\big)\bigr)}
  {\sum_{i',n',b'} w_{i'b'}^{n'}(t) \exp\bigl(\eta \big(\mathbb{E}_{q_{n'} \sim \hat{P}_{n'}(t)}[\ell_{\tau_{n'}, \theta_{i'}}(q_{n'}, Y_t)] 
  - \ell_{\tau_{n'}, \theta_{i'}}(f_{b'}^{n'}(t), Y_t)\big)\bigr)}.
  \end{align}
  \EndFor
  \end{algorithmic}
\end{algorithm}

A critical issue when implementing Algorithm \ref{alg:omni_online} is solving the minmax program (\ref{eq:minimax}) in an efficient manner. Our detailed procedure for accomplishing this is given in Appendix \ref{appsec:fast}. 
As a summary of those results, the following theorem states the key structural properties of our computed solution. 
Briefly, the first property states that although $\Pbhat(t)$ is randomized, this randomization is supported on only two nearby values; the second property states that the quantile forecasts produced by $\Pbhat(t)$ follow the anticipated order $q_1(t) \leq \cdots \leq q_N(t)$; and the third property gives a bound on the range of $\Pbhat(t)$'s support. We provide more details and the proof of this theorem in Appendix \ref{pf:Pn_property}. 

\begin{theorem}\label{thm:Pn_property} 
There exists a solution for $\Pbhat(t) = (\hat{P}_n(t))_{n=1}^N$ in Algorithm~\ref{alg:omni_online} such that:
\begin{enumerate}
    \item For all $n \in [N]$, $\hat{P}_n(t)$ is supported on a maximum of two values.
    \item There exists a coupling of $(\Phat_1(t),..., \Phat_n(t))$ such that any sampled value $\mathbf{q}(t) \sim \Pbhat(t)$ is ordered as $q_1(t) \leq \cdots \leq q_N(t)$. 
    \item For all $n \in [N]$, $\Phat_n(t)$ is supported on  $[\min_{n \in [N], b \in [B]} f_b^n(t), \,\, \max_{n \in [N], b \in [B]} f_b^n(t)]$. 
\end{enumerate}
\end{theorem}

Theorem \ref{thm:Pn_property} demonstrates that $\Pbhat(t)$ produces sensible quantile forecasts with minimal randomization. 

Our second key result shows that $\Pbhat(t)$ can be computed in a highly efficient manner. This is crucial because although the minimax program in (\ref{eq:minimax}) is convex-concave, it involves optimization over a high-dimensional parameter space consisting of $N$ distributions on $[m]$, making standard solution methods computationally intractable. As above, our detailed method for computing $\Pbhat(t)$ is deferred to Appendix \ref{appsec:fast}, {where we also provide an empirical runtime analysis that confirms the linear scaling in $m$, $N$, and $B$.}

\begin{proposition}\label{prop:omni_subroutine}
At each step of Algorithm \ref{alg:omni_online}, $\Pbhat(t)$ can be computed in $O(mNB)$ time. Consequently, Algorithm~\ref{alg:omni_online} can be implemented in $O(mNB)$ time complexity per step.
\end{proposition}

{
The next proposition shows that Algorithm \ref{alg:omni_online}'s forecasts depend on the discretization only through its resolution $1/m$, not through the interval one starts with. Consequently, the data range need not be known in advance: one may start on any interval and, whenever an out-of-range value appears, enlarge the grid. Additional detail and the proof is deferred to Appendix \ref{pf:grid_consistency}.

\begin{proposition}[Consistency under grid expansion]
\label{prop:grid_consistency}
Consider two runs of Algorithm \ref{alg:omni_online} on the same data stream and at the same resolution $1/m$, differing only in their initial discretization interval, the second being an enlargement of the first. Suppose that up to round $t$ every forecast $f_b^n(s)$ and every target $Y_s$ ($s\le t$) lies in the smaller interval. Then the two runs produce identical outputs $\Pbhat(s)$ for all $s\le t$.
\end{proposition}
}

\subsection{Comparison with Separate Optimization of Quantile Levels} \label{subsec:separate_multi_q}

Here, we illustrate an advantage of optimizing multiple quantile levels jointly as in Algorithm~\ref{alg:omni_online}. In particular, in light of Theorem~\ref{thm:decomp_multi_q}, one might ask whether running a separate omniprediction algorithm, analogous to Algorithm~\ref{alg:omni_online}, for each quantile level individually would suffice. In that case, the minimax program~\eqref{eq:Pn_def_minimax} would be replaced by 
\begin{align*}
\hat{P}_n(t) = \argmin_{P \in \Delta_m} \max_{P^y \in \Delta_m} \sum_{i=1}^m \sum_{b=1}^B w_{ib}^n\E_{q_n\sim P, Y \sim P^y}[\ell_{\tau_n, \theta_i}(q_n,Y) - \ell_{\tau_n, \theta_i}(f_b^n(t),Y)], \quad \forall n \in [N].
\end{align*}
By Theorem~\ref{thm:Pn_property}, Algorithm~\ref{alg:omni_online} outputs a vector of quantile estimates that is ordered according to the quantile levels. In contrast, no such guarantee holds when each quantile level is optimized separately. Moreover, simple procedures such as post-sorting the resulting quantile estimates may increase the loss and thus fail to provide an omniprediction guarantee (see Appendix~\ref{appsec:post_sorting}).
\section{Empirical example}
\label{sec:experiments}

To illustrate Algorithm~\ref{alg:omni_online}, we consider a real epidemiological forecasting task from the COVID-19 forecasting hub, a repository of real-time forecasts made during the COVID-19 pandemic \cite{cramer2022covidhub}. In particular, we consider forecasting weekly aggregated COVID-19 hospitalizations across 51 state-level series (50 states + DC) and the U.S. aggregate (US-agg), at $N=23$ quantile levels and horizons of $h \in \{1,2,3,4\}$ weeks ahead. We evaluate performance across all location–horizon pairs (each defining a time series). During the pandemic, research groups throughout the country submitted forecasts to this repository. We use the predictions from $B=18$ of these groups as the base forecasts for our algorithm. We compare our method against two ensembles constructed from the submitted forecasts by researchers during the pandemic, COVIDhub-Ens and JHUAPL-SLPEns, as well as Algorithm~\ref{alg:wql_hedge} (Hedge-QL, in Appendix \ref{appsec:multi_wql}). Hedge-QL applies the Hedge algorithm independently at each quantile level, forming predictions as weighted averages of base forecasts with weights chosen to minimize the pinball loss. This method acts as a simple baseline for ensembling quantile forecasts against which we compare our omniprediction procedure. We use 127 weeks of data from December 2020 to June 2023. Additional details about the data can be found in Appendix~\ref{appsubsec:experiment_details}.

\vspace{-0.2cm}
\paragraph{Metrics.}
We evaluate performance using three metrics. The first is the approximate cumulative omniprediction error at time \(t\) (referred to as the omni-error at time \(t\)):
\begin{align}\label{eq:omni_error_def}
\sup_{i,n,b} \frac{1}{t}\sum_{s=1}^t \left[\ell_{\tau_n, \theta_i}(q_n(s),Y_s) - \ell_{\tau_n, \theta_i}(f^n_b(s),Y_s)\right].
\end{align}
Note that Lemma~\ref{lem:discretize_multi} shows that this quantity approximates the true omniprediction error (Eq.~\eqref{eq:omni}) within a factor of 2. The second is the cumulative quantile loss at time \(t\), defined as
\begin{align*}
\frac{1}{Y_{\max}}\frac{1}{t}\sum_{s=1}^t \sum_{n=1}^N \mathbb{E}_{q_n \sim \hat{P}_n(s)} \big[\rho_{\tau_n}(Y_s - q_n)\big],
\end{align*}
where $\rho_{\tau}(x) = \tau x - \min\{x,0\}$ is the canonical pinball loss for the $\tau$-quantile and $Y_{\text{max}}$ is the maximum value of the ground truth in the given time-series. To clarify terminology, we refer to the omni-error and cumulative quantile loss at the final time step $T$ as the \textit{total omni-error} and \textit{total quantile loss}, respectively.  Finally, as our last metric, we also consider the instantaneous normalized quantile loss at time \(t\): $\frac{1}{Y_t}\sum_{n=1}^N \mathbb{E}_{q_n \sim \hat{P}_n(t)} \big[\rho_{\tau_n}(Y_t - q_n)\big]$.

\vspace{-0.2cm}
\paragraph{Discretization.}
As introduced in Section~\ref{nssec:discretization}, we discretize both the predictions and the ground truth into a finite set of values. 
We use grid spacings from $\{2,5,10,20,50,100,500\}$ depending on the range of forecast values. 
This results in more than $500$ grid points for all $208 = 52 \times 4$ time-series, and more than $1{,}000$ grid points for over $150$ time-series. 
Additional details on the choice of discretization are provided in Appendix~\ref{appsubsec:experiment_details}.

\vspace{-0.2cm}
\paragraph{Learning rates.}
Our algorithm and Hedge-QL each use a learning rate hyperparameter. Rather than tuning it explicitly, we evaluate each method over a range of learning rates, and we report results for the best-performing value on each metric, along with a sensitivity analysis showing how performance varies with learning rate. In plots and legends, the best performing learning rates are denoted by \texttt{*Our~algorithm} and \texttt{*Hedge-QL}. 
For some plots, we also report results using a fixed learning rate. 
Finally, consistent with our theory, when referring to an instance of our algorithm with learning rate constant $c_O$, this corresponds to an implemented learning rate of $\eta = c_O \sqrt{(\log mNB) / T}$. 
Similarly, for Hedge-QL with constant $c_H$, the learning rate is set as $\eta = c_H \sqrt{(\log B) / T} / Y_{\max}$. We evaluate our algorithm over \(c_O \in \{10^{i/4} \mid i=-8,\ldots,5\}\) and Hedge-QL over \(c_H \in \{10^{i/4} \mid i=-4,\ldots,12\}\) (note that larger constants are needed for Hedge-QL to account for the normalization by 1/$Y_{\max}$ which gives it a smaller effective learning rate).

\paragraph{Results.}
We present plots of omni-error over time, as well as total omni-error and total quantile loss for our algorithm and baseline ensembles. We also provide a sensitivity analysis with respect to the learning rate. Overall, the results show that our algorithm is comparable to or outperforms the baseline ensembles, with particularly strong performance on more difficult prediction tasks. Additional plots, including comparisons of quantile forecast trajectories and distributions of ranks of each model at each (horizon, location, time) tuple, are provided in Appendix~\ref{appsubsec:add_plots}.

\begin{figure}[t] 
  \centering
  \includegraphics[width=\textwidth]{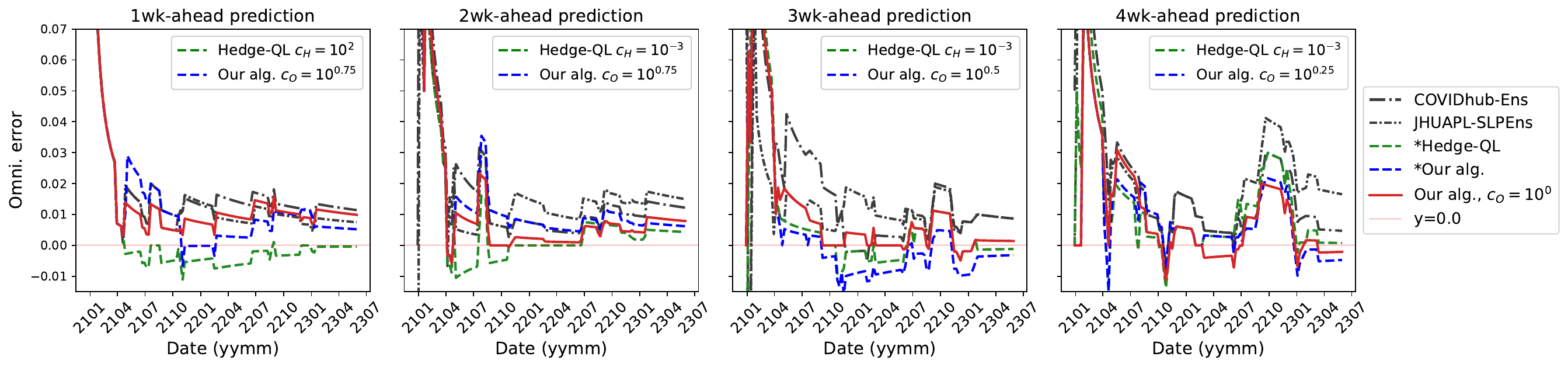}
  \caption{
  Omni-error (Eq. \eqref{eq:omni_error_def}) over time for $1$- to $4$-week-ahead forecasting horizons (left to right panels) on the US-agg data for \texttt{*Our algorithm} (blue), \texttt{*Hedge-QL} (green) and the COVIDhub-Ens and JHUAPL-SLPEns ensembles (both black). The values of $c_H, c_O$ shown in the legend of each panel correspond to the learning rate which achieved the lowest total omni-error. The red line shows the results of our algorithm with a fixed learning rate constant $c_O = 1$. Lower values indicate better performance.}
  \label{fig:omni_error_over_time}
\end{figure}

As a first evaluation, Figure~\ref{fig:omni_error_over_time} shows the omni-error over time for our algorithm, Hedge-QL, and the two existing ensembling models on the US-agg data. We find that our algorithm outperforms the COVIDhub-Ens and JHUAPL-SLPEns ensembles across all time horizons, with larger performance gaps occurring at longer time horizons where forecasting is more challenging. On the other hand, Hedge-QL performs similarly to our algorithm, with the precise relative  performance varying across horizons and learning rates. As a sensitivity check, this figure shows results for our algorithm obtained with both the best performing learning rate in hindsight and a fixed value of $c_O=1$. This latter value was chosen as it minimizes the theoretical bound given in Theorem \ref{thm:omni_guarantee}. Across all time horizons, forecasts from our algorithm with $c_O=1$ perform as well as or better than the baseline ensembles.

To evaluate the generalizability of these results, the left panel of Figure~\ref{fig:combined_results} shows the differences in total omni-error (left column) and total quantile loss (right column) between our algorithm and each of the three comparators (rows).  
Similar to the results for US-agg, we find that our algorithm consistently outperforms the COVIDhub-Ens and JHUAPL-SLPEns across all time horizons with greater gaps being observed for longer horizons. In comparison to Hedge-QL, we find that our algorithm has a small advantage in terms of the total omni-error and realizes a much bigger performance improvement in terms of the total quantile error, especially at three and four week horizons. 

The preceding results indicate that our algorithm shows the biggest performance gains on larger time horizons where forecasting is most difficult. To further investigate this trend, the right panel of Figure~\ref{fig:combined_results} displays the instantaneous quantile losses realized by our algorithm and Hedge-QL at each timestep. The color spectrum in this plot represents the difficulty of prediction at a given (horizon, location, time) tuple as measured by the average pinball loss of the base forecasters,
\(
\frac{1}{Y_t} \frac{1}{BN} \sum_{n=1}^N \sum_{b=1}^B \rho_{\tau_n}(Y_t - f_b^n(t)).
\)
We find that more difficult (darker red) points are predominantly located above the $y=x$ line indicating that our algorithm achieves lower quantile loss than Hedge-QL on harder prediction instances.

\begin{figure}[htbp]
  \centering

  \begin{minipage}{0.52\textwidth}
    \centering
    \includegraphics[width=\linewidth]{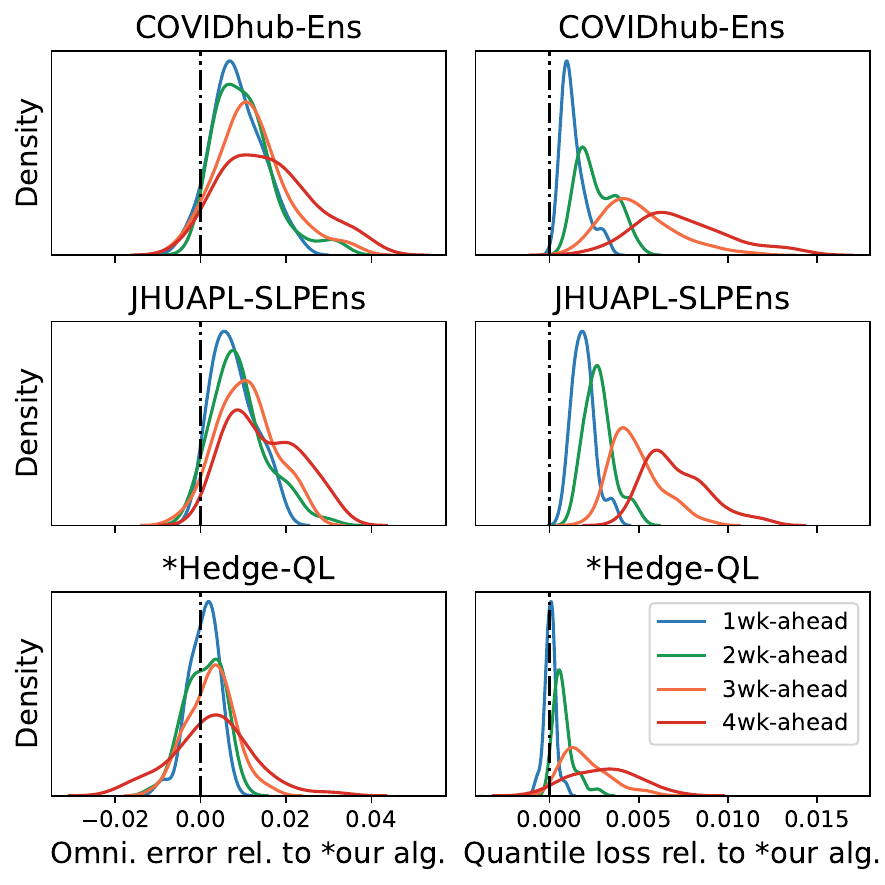}
  \end{minipage}
  \hfill
  \begin{minipage}{0.45\textwidth}
    \centering
    \includegraphics[width=\linewidth]{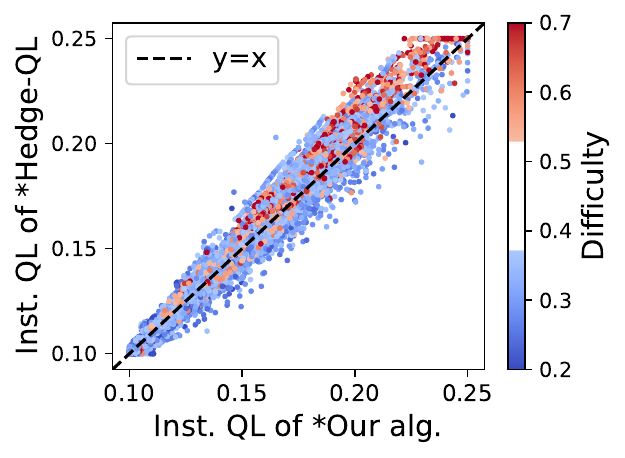}
  \end{minipage}

  \caption{
  \textbf{Left:} State-level comparison of performance relative to \texttt{*Our algorithm}. Rows show comparisons with three baseline ensembles. Each curve shows a Gaussian kernel density estimate of the difference in total omni-error (left column) or total quantile error (right column) across the $51$ state-level time series for a given forecasting horizon. Values to the right of zero (vertical black dotted line) indicate better performance of our algorithm.
  \textbf{Right:}  Comparison of instantaneous quantile loss realized by \texttt{*Our algorithm} (x-axis) and \texttt{*Hedge-QL} (y-axis). See text for detailed descriptions. 
  }
  \label{fig:combined_results}
\end{figure}


\begin{figure}[htbp]
  \centering
  \begin{subfigure}{0.98\textwidth}
      \centering
      \includegraphics[width=\textwidth]{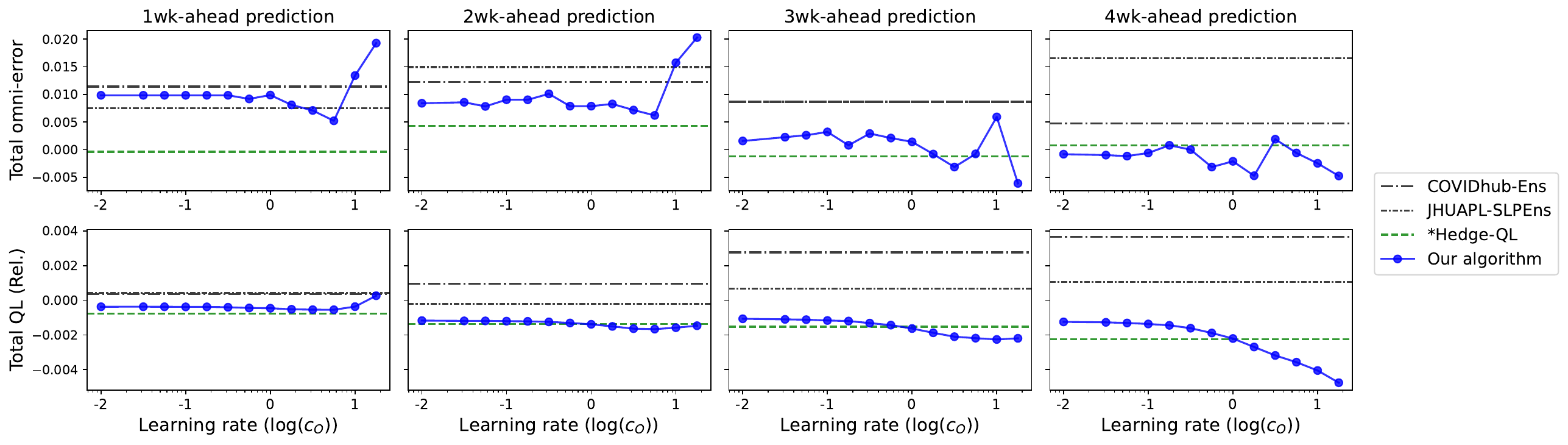}
  \end{subfigure}

  \caption{
    Sensitivity of our algorithm to the choice of learning rate. 
    Each panel shows the performance of our algorithm as a function of the learning rate (blue curve) as measured by the total omni-error (top row) or total quantile loss relative to best base forecaster (bottom row) on the US-agg data. Panels from left to right show results at $1$- to $4$-week time horizons, while horizontal lines  indicate the performance of three baseline methods. Lower values indicate better performance.
  }
  \label{fig:sensitivity}
\end{figure}

We conclude this section with a sensitivity analysis evaluating how the performance of our algorithm varies with the learning rate. Figure~\ref{fig:sensitivity} shows the realized total omni-error and total quantile loss realized by our algorithm for $1$- to $4$-week-ahead horizons on the US-agg data with learning rate constants $c_O \in \{10^{i/4} \mid i=-8,\ldots,5\}$.
We see that across almost all horizons and learning rates our algorithm outperforms COVIDhub-Ens and JHUAPL-SLPEns and matches or improves on the performance of Hedge-QL. In general, we find that the total quantile loss varies more smoothly with respect to the learning rate than the total omni-error, although neither metric appears to be significantly sensitive to its value. A similar sensitivity plot for Hedge-QL is given in Appendix \ref{appsubsec:add_plots}.  
\section{Limitations and future work}
\label{sec:discussion}

In this paper, we develop computationally and sample efficient omniprediction methods for continuous prediction targets. To make this problem tractable, our results focus on bounded proper losses for a finite set of quantile levels. While this setting captures a broad and practically relevant class of evaluation criteria, it does not cover unbounded losses or proper losses defined on the full predictive distribution. For example, the log-loss $\ell(f,y) = -\log f(y)$ falls into both categories and requires access to a forecasted density $f$. In general, quantile-based representations may be insufficient to evaluate such density-based proper losses, suggesting the need for new representations and algorithms beyond quantile aggregation. 
Exploring alternative aggregation criteria and determining the full scope of tractable omniprediction problems remain interesting directions for future work.

\begin{ack}
MWM would like to acknowledge DARPA, NSF, the DOE Competitive Portfolios grant, and the DOE SciGPT grant for partial support of this work.
\end{ack}

\bibliographystyle{alpha}
\bibliography{references}

\newpage
\appendix
\section{Proofs of Theoretical Results} 

\subsection{Proof of Lemma \ref{lem:discretize_multi}} \label{pf:discretize_multi}
\begin{proof}
  Let $\theta_{i^*} \in \{\theta_i\}_{i=1}^{m}$ be the closest value  from $\theta$, where the ties are broken by rounding up. Observe that $\one(\theta < a) = \one(\theta_{i^*} < a)$ for $a \in \{\frac{j}{m}\}_{j=0}^m$. This gives 
  \begin{align*}
  \ell_{\tau, \theta}(x,y) & = (\one(y<x)-\tau)(\one(\theta<x)-\one(\theta<y)) \\
  & = (\one(y<x)-\tau)(\one(\theta_{i^*}<x)-\one(\theta_{i^*}<y)) = \ell_{\tau, \theta_{i^*}}(x,y),
  \end{align*}
  for $x,y \in \{\frac{j}{m}\}_{j=0}^m$ and $\theta \in [0,1)$ and $\ell_{\tau, 1}(x,y) = 0$. Putting this together with characterization of $\ell \in \cl_1$ in Section \ref{ssec:characterization}, we obtain
  \begin{align*}
  & \sup_{\ell \in \cl_1, f\in\cf} \frac{1}{T} \sum_{t=1}^T \left[\ell(q(t),Y_t) - \ell(f(t),Y_t)\right] \\
  &= \sup_{H_n:\text{pos. measure}, f\in\cf} \frac{1}{T} \sum_{t=1}^T \sum_{n=1}^N \left[\int_0^1 \ell_{\tau_n, \theta}(q_n(t), Y_t) - \ell_{\tau_n, \theta}(f^n(t),Y_t) dH_n(\theta) \right] \\
  &\underset{\ell_{\tau, 1} \equiv 0}{=} \sup_{H_n, f} \frac{1}{T} \sum_{t=1}^T \sum_{n=1}^N \sum_{i=1}^{m} \bigl[\ell_{{\tau_n, \theta_i}}(q_n(t), Y_t) - \ell_{{\tau_n, \theta_i}}(f^n(t),Y_t)\bigr] \, 
  H_n\left(\theta \in \left[\theta_i-\frac{1}{2m}, \theta_i +\frac{1}{2m}\right)\!\right) \\
&= \sup_{H_n, f} \sum_{n=1}^N \sum_{i=1}^{m} \frac{1}{T} \sum_{t=1}^T\bigl[\ell_{{\tau_n, \theta_i}}(q_n(t), Y_t) - \ell_{{\tau_n, \theta_i}}(f^n(t),Y_t)\bigr] \, 
  H_n\left(\theta \in \left[\theta_i-\frac{1}{2m}, \theta_i +\frac{1}{2m}\right)\!\right) \\
  &\leq \sup_{i\in[m], n \in [N], f\in\cf} \frac{1}{T} \sum_{t=1}^T\bigl[\ell_{{\tau_n, \theta_i}}(q_n(t), Y_t) - \ell_{{\tau_n, \theta_i}}(f^n(t),Y_t)\bigr] \, \left(\sum_{n=1}^ NH_n([0,1))\right) .
  \end{align*}
  Now we compute an upper bound of $\left(\sum_{n=1}^ NH_n([0,1))\right)$ in terms of maximum possible value of $\ell((x_n)_{n=1}^N,y)$. Observe that
  \begin{align*}
  \ell(0,...,0,1) 
  &= \sum_{n=1}^N \ell_n(0,1) = \sumn \tau_n(g_n(1)-g_n(0)) \underset{\text{Thm.}\, \ref{thm:decomp_single_q}}{=} \sumn \tau_n H_n([0,1)) \\    
  \ell(1,...,1,0) 
  &= \sum_{n=1}^N \ell_n(1,0) = \sumn (1-\tau_n)(g_n(1)-g_n(0)) \underset{\text{Thm.}\, \ref{thm:decomp_single_q}}{=} \sumn (1-\tau_n) H_n([0,1)).
  \end{align*}
  Note that $g_n(1)-g_n(0) = H_n([0,1))$ from left-continuity of $g_n$. Since $\max\ell$ is larger than $\ell(0,...,0,1)$ and $\ell(1,...,1,0)$, it's also larger than the average of these two, which is $(1/2)\sum_{n=1}^ NH_n([0,1))$. Putting together with the condition $|\ell| \leq 1$, we have $\sum_{n=1}^ NH_n([0,1)) \leq 2$, and this proves 
  \[
  \frac{\sup_{\ell \in \cl_1, f\in\cf} \frac{1}{T}\sum_{t=1}^T \left[\ell(q(t),Y_t) - \ell(f(t),Y_t)\right]}
{\sup_{i,n,f} \frac{1}{T}\sum_{t=1}^T \left[\ell_{\tau_n, \theta_i}(q_n(t),Y_t) - \ell_{\tau_n, \theta_i}(f^n(t),Y_t)\right]}
\leq 2.
  \]
  It remains to prove that the ratio above is at least $1$. It follows from the fact that $\ell_{\tau_n, \theta_i} \in \cl_1$.
\end{proof}

\subsection{Proof of Theorem \ref{thm:Pn_property}} \label{pf:Pn_property}
In this proof, we denote $\Pbhat(t) = (\Phat_n(t))_{n=1}^N$ by $(P_n^*)_{n=1}^N$. We introduce one lemma and two propositions that are essential for proving that the minimax solution $P_n^*$ satisfies the monotonicity constraint (Property $2$ in the theorem statement). First, Lemma~\ref{lem:jn_sort} defines a quantity $V_n^*$, on which our closed-form solution $P_n^*$ depends, for each $n \in [N-1]$, and shows that there exists a non-decreasing sequence $\{j_n^*\}_{n=1}^{N-1}$, where each $j_n^*$ is an index that minimizes $V_n(j)$. 

\begin{lemma}\label{lem:jn_sort}
Denote $[m]_0 := [m] \cup \{0\}$ and define a function $V_n(j)$ and a quantity $V_n^*$ for $n \in [N-1]$ as
\begin{align} \label{eq:Vn_def}
V_n(j) &:= \sum_{i=j+1}^{m} \sum_{s=1}^{n} \sum_{b=1}^B w_{ib}^s \indi{\theta_i < f_b^s(t)} + \sum_{i=1}^{j} \sum_{s=n+1}^{N}  \sum_{b=1}^b w_{ib}^s (1 - \indi{\theta_i < f_b^s(t)}) \geq 0, \nonumber\\
V_n^* &:= \min_{j \in [m]_0} V_n(j) ,
\end{align}
and $V_0^* = V_N^* := 0$.
Let $J_n^*$ be the set of indices $j \in [m]_0$ that minimizes the $V_n(j)$. 
\smallskip \\
Then, there exists a sorted sequence $0=j_0^* \leq j_1^* \leq ... \leq j_{N-1}^* \leq j_N^* = m$ such that $j_n^* \in J_n^*$ for $n \in [N-1]$.
\end{lemma}

\begin{proof}[Proof of Lemma \ref{lem:jn_sort}]
Define $w_i^s := (w_{ib}^s)_{b=1}^B$, $\one_i^s := (\indi{\theta_i < f_b^s(t)})_{b=1}^B$, and $\one_B := (1, ..., 1) \in \R^B$. Then, $V_n^*$ can be rewritten as
\[
V_n^* = \min_{j \in [m]_0} \left( \sum_{i=j+1}^{m} \sum_{s=1}^{n} w_i^s \cdot \one_i^s + \sum_{i=1}^{j} \sum_{s=n+1}^{N}  w_i^s \cdot (\one_B - \one_i^s) \right).
\]
Suppose there exists a pair of indices $n, n+1$ such that $j_n^* > j_{n+1}^*$. Using the definition of $V_n^*$, we have 
\begin{align*}
j_n^* &= \argmin_{j \in [m]_0} \left( \sum_{i=j+1}^{m} \sum_{s=1}^{n} w_i^s \cdot \one_i^s + \sum_{i=1}^{j} \sum_{s=n+2}^{N}  w_i^s \cdot (\one_B - \one_i^s) + \sum_{i=1}^{j} w_i^{n+1} \cdot (\one_B - \one_i^{n+1})\right), \\
j_{n+1}^* &= \argmin_{j \in [m]_0} \left( \sum_{i=j+1}^{m} \sum_{s=1}^{n} w_i^s \cdot \one_i^s + \sum_{i=1}^{j} \sum_{s=n+2}^{N}  w_i^s \cdot (\one_F - \one_i^s) + \sum_{i=j+1}^{m} w_i^{n+1} \cdot \one_i^{n+1}\right).
\end{align*}
Two expressions share the first two terms and differ in the last term. We abbreviate these terms as the following:
\begin{align*}
C_j & := \sum_{i=j+1}^{m} \sum_{s=1}^{n} w_i^s \cdot \one_i^s + \sum_{i=1}^{j} \sum_{s=n+2}^{N}  w_i^s \cdot (\one_B - \one_i^s) \\
D_j & := \sum_{i=1}^{j} w_i^{n+1} \cdot (\one_B - \one_i^{n+1}) \\
E_j & := \sum_{i=j+1}^{m} w_i^{n+1} \cdot \one_i^{n+1}.
\end{align*}
Note that $D_j$ is non-decreasing in $j$ and $E_j$ is non-increasing in $j$.
Since $j_n^*$ and $j_{n+1}^*$ are the indices that minimize the $V_n(j)$ and $V_{n+1}(j)$, respectively, we have $C_{j_n^*} + D_{j_n^*} \leq C_{j_{n+1}^*} + D_{j_{n+1}^*}$ and $C_{j_n^*} + E_{j_n^*} \geq C_{j_{n+1}^*} + E_{j_{n+1}^*}$. This implies that 
\[
E_{j_{n+1}^*} - E_{j_n^*} \leq C_{j_n^*} - C_{j_{n+1}^*} \leq D_{j_{n+1}^*} - D_{j_n^*}.
\]
Because of the monotonicity of $D_j$ and $E_j$ and the assumption that $j_n^* > j_{n+1}^*$, the leftmost term should be non-negative and the rightmost term should be non-positive. Therefore, we obtain 
\[
C_{j_n^*} - C_{j_{n+1}^*} = D_{j_n^*} - D_{j_{n+1}^*} = E_{j_n^*} - E_{j_{n+1}^*} = 0 \quad \Rightarrow j_{n+1}^* \in J_n^*, \,\, j_n^* \in J_{n+1}^*.
\]
By choosing $j_{n+1}^*$ as a minimizer of $V_n(j)$ and $j_n^*$ as a minimizer of $V_{n+1}(j)$, we resolve the unsorted pair of indices. Repeating this process yields a sorted sequence $0=j_0^* \leq j_1^* \leq ... \leq j_{N-1}^* \leq j_N^* = m$ and this proves the existence of sorted sequence.
\\
Due to its existence, we can simply obtain a sorted sequence by choosing $j_n^* = \min J_n^*$ for all $n \in [N]$.
\end{proof}

\bigskip

Next, we introduce two propositions. The first proposition (Proposition \ref{prop:Pn_def}) constructs $(P_n^*)_{n=1}^N$ and prove their properties. The second proposition (Proposition \ref{prop:Pn_is_minimax_soln}) then show this $(P_n^*)_{n=1}^N$ is indeed a solution for the minimax program \eqref{eq:Pn_def_minimax} in Algorithm \ref{alg:omni_online}. Note that Property 1 in \ref{prop:Pn_def} statement below proves Property 1 in the original theorem (Theorem \ref{thm:Pn_property}) statement. 

\begin{proposition} \label{prop:Pn_def}
Given a sorted sequence $0=j_0^* \leq j_1^* \leq ... \leq j_{N-1}^* \leq j_N^* = m$ from Lemma \ref{lem:jn_sort}, there exists a distribution $P_n^*$ for $n \in [N]$ of the form
\begin{align}\label{eq:Pn_def}
P_n^* = (1-C_n)\delta_{k_n / m} + C_n\delta_{(k_n+1)/m},
\end{align}
such that the following properties hold:
\begin{enumerate}
  \item $j_{n-1}^* \leq k_n$
  \item $k_n < j_n^*$ if $j_{n-1}^* \neq j_n^*$ and $k_n = j_n^*$ with $C_n=0$ if $j_{n-1}^* = j_n^*$.
  \item The support of $P_n^*$ is contained in $[j_{n-1}^*/m, j_n^*/m]$.
  \item $\E_{q_n \sim P_n^*} \left[ \sum_{i=1}^{m} \sum_{b=1}^B w_{ib}^n \left( \indi{\theta_i < q_n} - \indi{\theta_i < f_b^n(t)} \right) \right] = V_{n-1}^*-V_n^*$.
\end{enumerate}
\end{proposition}

\begin{proof}[Proof of Proposition \ref{prop:Pn_def}]
  We use the notation $w_i^s$, $\one_i^s$, and $\one_B$ from the proof of Lemma~\ref{lem:jn_sort}. 
  Note that the quantity $V_n^*$ defined in Lemma \ref{lem:jn_sort} is equal to $V_n(j_n^*)$ for all $n \in [N-1]$ and that $V_n(j) \geq V_n(j_n^*)$ for all $j \in [m]_0$.
  \\
  Using the fact that $\indi{\theta_i < k/m} = \indi{i \leq k}$, the last property in the proposition statement is equivalent to 
  \begin{align*} 
  & V_{n-1}^* - V_n^* \\
  & = \E_{q_n \sim P_n^*} \left[ \sum_{i=1}^{m} \sum_{b=1}^B w_{ib}^n \left( \indi{\theta_i < q_n} - \indi{\theta_i < f_b^n(t)} \right) \right] \nonumber\\
  & = (1-C_n) \sum_{i=1}^{m} \sum_{b=1}^B w_{ib}^n \left( \indi{i \leq k_n} - \indi{\theta_i < f_b^n(t)} \right) \\
  & \qquad\qquad\qquad + C_n \sum_{i=1}^{m} \sum_{b=1}^B w_{ib}^n \left( \indi{i \leq k_n+1} - \indi{\theta_i < f_b^n(t)} \right) \nonumber\\
  & = (1-C_n) \sum_{i=1}^{m} w_{i}^n \cdot \left( \indi{i \leq k_n} \one_B - \one_i^n \right) + C_n \sum_{i=1}^{m} w_{i}^n \cdot \left( \indi{i \leq k_n+1} \one_B - \one_i^n \right) \nonumber\\
  & = \sum_{i=1}^{k_n} w_i^n \cdot \one_B + C_n w_{k_n}^n \cdot \one_B - \sum_{i=1}^{m} w_i^n \cdot \one_i^n.
  \end{align*}
  The equality between the first and the last line is equivalent to
  \begin{align}
  \sum_{i=1}^{k_n} w_i^n \cdot \one_B + C_n w_{k_n}^n \cdot \one_B
  = \sum_{i=1}^{m} w_i^n \cdot \one_i^n + V_{n-1}^* - V_n^* 
  = \sum_{i=1}^{j_{n-1}^*} w_i^n \cdot \one_B + V_n(j_{n-1}^*) - V_n(j_n^*), \label{eq:Pn_eq}
  \end{align}
  where we used the identity $\sum_{i=1}^{m} w_i^n \cdot \one_i^n + V_{n-1}(j) = \sum_{i=1}^{j} w_i^n \cdot \one_B + V_n(j)$ evaluated at $j=j_{n-1}^*$.
  \smallskip \\
  Now we divide the cases into three cases. 
  First, consider the case where $j_{n-1}^* = j_n^*$. In this case, simply $k_n = j_{n-1}^*$ with $C_n=0$ satisfies the equation \eqref{eq:Pn_eq}. Therefore, $P_n^* = \delta_{j_{n-1}^*/m}$ satisfies the conditions.
  Second, consider the case where $j_{n-1}^* < j_n^*$ but $V_n(j_{n-1}^*) = V_n(j_n^*)$. $k_n = j_{n-1}^*$ with $C_n=0$ again satisfies the equation \eqref{eq:Pn_eq} and $k_n = j_{n-1}^* < j_n^*$ in this case.
  Lastly, consider the case where $V_n(j_{n-1}^*) > V_n(j_n^*)$. We have 
  \begin{align*}
    \sum_{i=1}^{j_{n-1}^*} w_i^n \cdot \one_F
    & < \sum_{i=1}^{j_{n-1}^*} w_i^n \cdot \one_F + V_n(j_{n-1}^*) - V_n(j_n^*) \\
    & = \sum_{i=1}^{m} w_i^n \cdot \one_i^n + V_{n-1}^* - V_n^* \\
    & = \sum_{i=1}^{j_{n}^*} w_i^n \cdot \one_F + V_{n-1}(j_{n-1}^*) - V_{n-1}(j_n^*) 
    \leq \sum_{i=1}^{j_{n}^*} w_i^n \cdot \one_F,
  \end{align*}
  
  where two equalities used the identity $\sum_{i=1}^{m} w_i^n \cdot \one_i^n + V_{n-1}(j) = \sum_{i=1}^{j} w_i^n \cdot \one_B + V_n(j)$ evaluated at $j=j_n^*-1$ and $j=j_n^*$ respectively.
  This implies that there exists a $j_{n-1}^* \leq k_n < j_n^*$ such that 
  \begin{align}\label{eq:Pn_kn_condition}
    \sum_{i=1}^{k_n} w_i^n \cdot \one_B < \sum_{i=1}^{m} w_i^n \cdot \one_i^n - V_n^* + V_{n-1}^* \quad \text{and} \quad \sum_{i=1}^{k_n+1} w_i^n \cdot \one_B \geq \sum_{i=1}^{m} w_i^n \cdot \one_i^n - V_n^* + V_{n-1}^*.
  \end{align}
  With such $k_n$, we define $P_n^*$ with 
  \begin{align}\label{def:cn}
  C_n = \frac{1}{w_{k_n+1}^n \cdot \one_B} \left( \sum_{i=1}^{m} w_i^n \cdot \one_i^n - V_n^* + V_{n-1}^* - \sum_{i=1}^{k_n} w_i^n \cdot \one_B \right),
  \end{align}
  and this satisfies $\text{supp}(P_n^*) \in [j_{n-1}^*/m, j_n^*/m]$ since $k_n+1 \leq j_n^*$.
\end{proof}

\bigskip

The following theorem shows that $(P_n^*)_{n=1}^N$ is a solution for the minimax program \eqref{eq:Pn_def_minimax} in Algorithm \ref{alg:omni_online}.

\begin{proposition}\label{prop:Pn_is_minimax_soln} 
Given quantile levels $\tau_1 < \tau_2 < ... < \tau_N$ and quantile forecasts ${f}_b^n(t)$ for $b \in [B]$, $n \in [N]$, define $V_n^*$ for $n \in [N-1]$ as in equation \eqref{eq:Vn_def}, with $V_0^* = V_N^* := 0$. Due to Lemma \ref{lem:jn_sort}, there exists an increasing sequence $\{j_n^*\}_{n=1}^{N-1}$ where $j_n^*$ is an index that minimizes $V_n(j)$. Define $j_0^* = 0$ and $j_N^* = m$.
\smallskip \\
Then, $P_n^*$ for $n \in [N]$ given by \eqref{eq:Pn_def} with 
\[
P^{y*} = \tau_1 \delta_0 + (1-\tau_N)\delta_1 + \sum_{n=1}^{N-1}(\tau_{n+1}-\tau_n)\delta_{j_n^*/m}
\]
is a solution of the minimax program \eqref{eq:Pn_def_minimax} and the value of the minimax program is 
\[
-\sum_{n=1}^{N-1}(\tau_{n+1}-\tau_n)V_n^* \leq 0.
\]
\end{proposition}

\begin{proof}[Proof of Proposition \ref{prop:Pn_is_minimax_soln}]
  Denote the minimax objective function by 
  \[O\left((P_n)_{n=1}^N, P^y\right) = \sum_{n=1}^N \sum_{i=1}^{m} \sum_{b=1}^B w_{ib}^n \E_{q_n \sim P_n, Y \sim P^y}[\ell_{\tau_n, \theta_i}(q_n,Y) - \ell_{\tau_n, \theta_i}(f_b^n(t),Y)].
  \]
It suffices to show that $P^{y*}$ maximizes $O((P_n^*)_{n=1}^N,P^y)$ over $P^y$, and that $(P_n^*)_{n=1}^N$ minimizes $O((P_n)_{n=1}^N,P^{y*})$ over $P_n\in\Delta_m$, $n\in[N]$.
We start with decomposing $O\left((P_n)_{n=1}^N, P^y\right)$. First, observe the following equality:
\begin{align}
\label{eq:oppy_q}
&\ell_{\tau_n, \theta_i}(q_n,Y) - \ell_{\tau_n, \theta_i}(f_b^n(t), Y) \nonumber\\
&= (\one(Y<q_n) - \tau_n)\big(\one(\theta_i<q_n)-\one(\theta_i<Y)\big) \nonumber\\
& \qquad - (\one(Y<f_b^n(t))-\tau_n)\big(\one(\theta_i<f_b^n(t))-\one(\theta_i<Y)\big) \nonumber\\
&= \one(Y, \theta_i < q_n) - \tau_n \one(\theta_i<q_n) - \one(\theta_i<Y<q_n) \nonumber\\
& \qquad - \one(Y, \theta_i < f_b^n(t)) + \tau_n \one(\theta_i<f_b^n(t)) + \one(\theta_i<Y<f_b^n(t)) \nonumber\\
&= \one(Y \le \theta_i < p) - \tau_n \one(\theta_i<q_n) - \one(Y \le \theta_i < f_b(t))  + \tau_n \one(\theta_i < f_b^n(t)).
\end{align}

Then, $O\left((P_n)_{n=1}^N, P^y\right)$ can be written as
  \begin{align}
  & O\left((P_n)_{n=1}^N, P^y\right) \nonumber\\
  & = \sum_{n=1}^N \sum_{i=1}^{m} \sum_{b=1}^B w_{ib}^n \E_{q_n \sim P_n, Y \sim P^y}[\one(Y \leq \theta_i < q_n) - \tau_n \one(\theta_i<q_n) - \one(Y \leq \theta_i < f_b^n(t))] + C_n,
  \end{align}
  where $C_n =\tau_n \sum_{i=1}^{m} \sum_{b=1}^B w_{ib}^n \indi{\theta_i < f_b^n(t)}$ which does not depend on $P_n$ or $P^y$.
  
  Using this decomposition, we prove $O\left((P_n^*)_{n=1}^N, P^{y*}\right) = \max_{P^y} O\left((P_n^*)_{n=1}^N, P^y\right)$. Let $p_{yj}=P^y(\frac{j}{m})$ for $j=0,...,m$. Then, again using the notation $w_i^s$, $\one_i^s$, and $\one_B$ from the proof of Lemma~\ref{lem:jn_sort}, $O\left((P_n^*)_{n=1}^N, P^y\right)$ is rephrased as the following:
  \begin{align}\label{eq:O_Pstar_Py_decomp}
  O\left((P_n^*)_{n=1}^N, P^y\right) 
  & = \sum_{n=1}^N \E_{q_n \sim P_n^*} \left[\sum_j p_{yj} \sum_{i=j+1}^{m} w_i^n \cdot \left( \indi{\theta_i < q_n} \one_F - \one_i^n \right)\right] + C(P_n^*),
  \end{align}
  where $C(P_n^*) = - \tau_n \E_{q_n \sim P_n^*} \sum_{i=1}^{m} w_i^n \cdot (\indi{\theta_i < q_n} \one_F - \one_i^n)$ only depends on $P_n^*$ and does not depend on $P^y$.
  
  Plugging in $P_n^*$ given as \eqref{eq:Pn_def}, the summand can be expressed as 
  \begin{align*}
  A_j^n 
  &:= \E_{q_n \sim P_n^*} \left[\sum_{i=j+1}^{m} w_i^n \cdot \left( \indi{\theta_i < q_n} \one_F - \one_i^n\right)\right]  \\
  &= \begin{cases}
  -\sum_{i=1}^{j} w_i^n \cdot (\one_B-\one_i^n) - V_n^* + V_{n-1}^* & \quad j \leq k_n^* \\
  - \sum_{i=j+1}^{m} w_i^n \cdot \one_i^n & \quad j > k_n^* .
  \end{cases}
  \end{align*}
  
  This leads to 
  \begin{align*}
  A_j & := \sum_{n=1}^N A_j^n \\
  &= \begin{cases}
  -\sum_{n=1}^N \sum_{i=1}^{j} w_i^n \cdot (\one_B - \one_i^n) + V_0^* & \quad j \leq k_1 \\
  -\sum_{s=1}^n \sum_{i=j+1}^{m} w_i^s \cdot \one_i^s & \\
  \qquad - \sum_{s=n+1}^{N} \sum_{i=1}^{j} w_i^s \cdot (\one_B-\one_i^s) + V_n^* & \quad  j \in (k_n, k_{n+1}], \,\, n \in [N-1] \\
  -\sum_{n=1}^N \sum_{i=j+1}^{m} w_i^n \cdot \one_i^n + V_N^* & \quad j > k_N .
  \end{cases} 
  \end{align*}
  Notice that $A_j$ is non-increasing on $j \leq k_1$ and non-decreasing on $j > k_N$. $A_0 = 0$ and $A_m = 0$ is the maximum value of $A_j$ in these ranges, respectively. Next, observe that $A_j \leq 0$ for all $j=0,1,...,m$ since $V_n^*$ is the minimum value that $\sum_{s=1}^n \sum_{i=j+1}^{m} w_i^s \cdot \one_i^s + \sum_{s=n+1}^{N} \sum_{i=1}^{j} w_i^s \cdot (\one_B-\one_i^s)$ can achieve by its definition. It suffices to show that $A_{j_n^*} = 0$ for $n \in [N-1]$ since $P^{y*}$ only puts mass on these points. Note that $A_{j_0} = A_0 = 0$ and $A_{j_N} = A_m = 0$.
  
  From the first and second properties in the statement of Proposition \ref{prop:Pn_def}, we have $k_n < j_n^* \leq k_{n+1}$ if $j_{n-1}^* \neq j_n^*$ and $j_{n-1}^* = j_n^*= k_n$ if $j_{\nu-1}^* = j_\nu^*$, for $n \in [N-1]$. In the former case, we obtain $A_{j_n^*} = 0$ from $j \in (k_n, k_{n+1}]$ case in the definition of $A_j$ above. In the latter case, $j_n^* = k_n$ does not fall into the range $(k_n, k_{n+1}]$. Still, from $j \in (k_{n-1}, k_{n}]$ case ($j \leq k_1$ case if $n=1$) in the definition of $A_j$ , we have
  \begin{align*}
  A_{j_n^*} 
  & = -\sum_{s=1}^{n-1} \sum_{i=j+1}^{m} w_i^s \cdot \one_i^s - \sum_{s=n}^{N} \sum_{i=1}^{j} w_i^s \cdot (\one_B - \one_i^s) + V_{n-1}^* \\
  & =  - V_{n-1}(j_n^*) + V_{n-1}^* = -V_{n-1}(j_{n-1}^*) + V_{n-1}^* = 0,
  \end{align*}
  where $V_n(j)$ is defined as equation \eqref{eq:Vn_def}. Therefore, $A_{j_n^*} = 0$ for all $n \in [N-1]$ and $O\left((P_n^*)_{n=1}^N, P^{y*}\right) = \max_{P^y} O\left((P_n^*)_{n=1}^N, P^y\right)$. \smallskip
  
  Next, we show that $O\left((P_n^*)_{n=1}^N, P^{y*}\right) = \min_{P_n \in \Delta_m, n\in[N]} O\left((P_n)_{n=1}^N, P^{y*}\right)$. This follows directly from the propriety of $\ell_{\tau_n,\theta_i}$ losses and the third property in Proposition \ref{prop:Pn_def}. Any point in the support of $P_n^*$ is $\tau_n$-level quantile of $P^{y*}$ since $P^{y*}(Y \leq j_{n-1}^*/m) = \tau_n$ and $P^{y*}(Y \geq j_n^*/m) = \tau_n$.
  
  Lastly, we compute the value of the minimax program. Decomposition \eqref{eq:O_Pstar_Py_decomp} can be rewritten as 
  \[
  O\left((P_n^*)_{n=1}^N, P^y\right) = \sum_{j=0}^m p_{yj} A_j + \sum_{n=1}^N C(P_n^*).
  \]
  Since $P^{y*}$ only puts mass on points $j/m$ such that $A_j = 0$, plugging in $P^{y*}$ into the equation above yields
  \begin{align*}
  O\left((P_n^*)_{n=1}^N, P^{y*}\right) 
  &= \sum_{n=1}^N C(P_n^*) = -\sum_{n=1}^N \tau_n \E_{q_n \sim P_n^*} \sum_{i=1}^{m} w_i^n \cdot (\indi{\theta_i < q_n}\one_F - \one_i^n) \\
  &  = -\sum_{n=1}^N \tau_n (V_{n-1}^* - V_n^*) = -\sum_{n=1}^{N-1}(\tau_{n+1}-\tau_n)V_n^* \leq 0,
  \end{align*}
  where the second from the last equality used the fourth property in Proposition \ref{prop:Pn_def}.
\end{proof}

\bigskip

Lastly, we prove the second and third property in Theorem \ref{thm:Pn_property} statment using Proposition \ref{prop:Pn_def} and Proposition \ref{prop:Pn_is_minimax_soln}. 
\begin{proof}[Proof of Property 2 in Theorem \ref{thm:Pn_property}]
  It suffices to prove that $P_{n+1}$ is stochastically larger than $P_n$ since we can couple these two by $U \sim \text{Unif}(0,1)$ and $P_n = F_{P_n}^{-1}(U)$ and $P_{n+1} = F_{P_{n+1}}^{-1}(U)$ where $F_X^{-1}$ denotes a inverse of cumulative distribution function of $X$. 
  
  Combining the first three properties in the statement in Proposition \ref{prop:Pn_def}, we have $k_n < j_n^* \leq k_{n+1}$ if $j_{n-1}^* \neq j_n^*$. In this case, $P_{n+1}^*$ is stochastically larger than $P_n^*$ since $P_n^*$ is supported on $\{k_n / m, (k_n+1)/m\}$ and $k_{n+1} > k_n$. In case $j_{n-1}^* = j_n^*$, $P_n^* \equiv j_{n}^*$ and $P_{n+1}$ is stochastically larger than $P_n$ since $k_{n+1} \geq j_n^*$.
\end{proof}

\bigskip

\begin{proof}[Proof of Property 3 in Theorem \ref{thm:Pn_property}]
  Since $P_n^*$ is stochastically ordered, it suffices to prove that $P_1^* \geq m_1 / m$ and $P_N^* \leq m_2 / m$ almost surely.
  First we prove that $P_1^* \geq m_1 / m$ almost surely. Observe that, for $i \leq m_1$, $\indi{\theta_i < f_b^n(t)} = 1$ for all $b \in [B]$ and $n \in [N]$. This implies that $\one_i^n = \one_B$ for all $n \in [N]$ and $i \leq m_1$. Therefore, for $j \leq m_1$,
  \[
  V_1(j) = \sum_{i=j+1}^m w_i^1 \cdot \one_i^1 + \sum_{i=1}^j \sum_{s=2}^N w_i^s \cdot (\one_B - \one_i^s) 
  = \sum_{i=j+1}^{m_1} w_i^1 \cdot \one_B + \sum_{i=m_1+1}^{m} w_i^1 \cdot \one_i^1
  \]
  holds and $V_1(j)$ is non-increasing in $j$ on $j \leq m_1$. This implies that $V_1^* \leq V_1(m_1) = \sum_{i=m_1+1}^m w_i^1 \cdot \one_i^1$ and thus 
  \[
  \sum_{i=1}^{m_1} w_i^1 \cdot \one_B 
  = \sum_{i=1}^{m_1} w_i^1 \cdot \one_i^1 
  \leq \sum_{i=1}^m w_i^1 \cdot \one_i^1 - V_1^* + V_0^*, 
  \] 
  since $\one_i^s = \one_B$ for $i \leq m_1$ and $V_0^* = 0$. 
  In case the inequality is strict, $k_1$ that satisfies the condition \eqref{eq:Pn_kn_condition} should be at least $m_1$, which proves $P_1^* \geq m_1 / m$ almost surely. In case the equality holds, $k_1+1 = m_1$ since $w_{m_1}^1 \cdot \one_B = \sum_{b=1}^B w_{m_1}^1 > 0$ from the corollary condition. However, $C_1 = 1$ in this case and $P_1^* = \delta_{m_1 / m}$ which satisfies $P_1^* \geq m_1 / m$ almost surely.
  
  Next, we prove that $P_N^* \leq m_2 / m$ almost surely in a similar manner. For $i > m_2$, $\indi{\theta_i < f_b^n(t)} = 0$ for all $b \in [B]$ and $n \in [N]$. This implies that $\one_i^n = \one_B$ for all $n \in [N]$ and $i > m_2$. Therefore, for $j \geq m_2$,
  \[
  V_{N_1}(j) = \sum_{i=j+1}^m \sum_{s=1}^{N-1} w_i^s \cdot \one_i^s + \sum_{i=1}^j w_i^N \cdot (\one_B - \one_i^N) 
  = \sum_{i=1}^{m_2} w_i^N \cdot (\one_B - \one_i^N) + \sum_{i=m_2+1}^{m} w_i^N \cdot \one_B
  \]
  holds and $V_{N-1}(j)$ is non-decreasing in $j$ on $j \geq m_2$. This implies that $V_{N-1} \leq V_{N-1}(m_2) = \sum_{i=1}^{m_2} w_i^N \cdot (\one_B - \one_i^N)$ and thus 
  \[
  \sum_{i=1}^m w_i^N \cdot \one_i^N - V_{N}^* + V_{N-1}^* 
  \leq \sum_{i=1}^m w_i^N \cdot \one_i^N + \sum_{i=1}^{m_2} w_i^N \cdot (\one_B - \one_i^N) 
  = \sum_{i=1}^{m_2} w_i^N \cdot \one_B
  \]
  since $\one_i^N = \mathbf{0}$ for $i > m_2$ and $V_N^* = 0$. Because $w_{m_2}^N \cdot \one_B = \sum_{b=1}^B w_{m_2}^N > 0$ from the corollary condition, $k_n$ that satisfies the condition \eqref{eq:Pn_kn_condition} should satisfy $k_N < m_2$, which proves $P_N^* \leq m_2 / m$ almost surely.
\end{proof}

\subsection{Proof of Theorem \ref{thm:omni_guarantee}}\label{pf:omni_guarantee}

\begin{theorem}[Full statement of Theorem \ref{thm:omni_guarantee}]
For a set of quantile levels $\tau_1 < \tau_2 < ... < \tau_N$, let $\cl_1$ denote the class of proper losses for multiple quantile levels defined as equation \eqref{def:cl}. Consider base forecasters $\{f_b\}_{b=1}^B$ where each forecaster $f_b$ provides $\{\tau_n\}_{n=1}^N$-level quantile forecasts $\{f_b^n(t)\}_{n=1}^N$ at each time $t \in [T]$. With the target values $Y_t$, forecaster predictions $f_b^n(t)$ for $t \in [T], n \in [N], b \in [B]$ each taking values in a uniform grid of size $m$, and the learning rate $\eta = c\sqrt{(\log mNB)/T}$, the output $\Pbhat(t) = (\hat{P}_n(t))_{n=1}^N$ from Algorithm \ref{alg:omni_online} achieves the following omniprediction error:
  \begin{align*}
  \sup_{\ell \in \cl_1,\, b \in [B]} \frac{1}{T} \sum_{t=1}^T \E_{q_n \sim \hat{P}_n(t), \forall n} [\ell((q_n)_{n=1}^N, Y_t)] - \ell(f_b(t), Y_t)]
  \leq 2\left(c+\frac{1}{c}\right)\left(\sqrt{\frac{\log mNB}{T}}\right)
  \end{align*}
\end{theorem}

\begin{proof}
Observe that $\ell_{\tau_n, \theta_i}(x,y) \in \{0, \tau_n, 1-\tau_n\}$ and thus $|\ell_{\tau_n, \theta_i}(x_1,y_1) - S_{\tau_n, \theta_i}(x_2,y_2)| \leq 1$. Without loss of generality, we assume that $Y_t, f_b(t) \in \{0, \frac{1}{m}, ..., 1\}$ for all $t \in [T], b \in [B]$, as this construction extends to any bounded interval $[-B_\ell, B_u]$ via rescaling. 

The regret bound of hedge algorithm (Theorem \ref{appthm:hedge_regret}) gives the following:
\begin{align}
  & \max_{i \in [m], n \in [N], b \in [B]} \frac{1}{T} \sum_{t=1}^T \E_{q_n \sim \hat{P}_n(t), \forall n}\left[ \ell_{\tau_n, \theta_i}(q_n, Y_t) - \ell_{\tau_n, \theta_i}(f_b^n(t), Y_t) \right] \nonumber\\
  & \leq \frac{1}{T} \sum_{t=1}^T \sum_{i=1}^{m} \sumn \sum_{b=1}^B w_{ib}^n(t) \, \E_{q_n \sim \hat{P}_n(t), \forall n}\left[ \ell_{\tau_n, \theta_i}(q_n, Y_t) - \ell_{\tau_n, \theta_i}(f_b^n(t), Y_t) \right] + \eta + \frac{\log mNB}{T\eta}  . 
  \label{ieq:hedge_w_update}
\end{align}

Proposition \ref{prop:Pn_is_minimax_soln} states that the value of the minimax program \eqref{eq:Pn_def_minimax} is non-positive. Putting these together with previous lemmas, we have:
\begin{align*}
  & \sup_{\ell \in \cl_1,\, f \in \cf} \frac{1}{T} \sum_{t=1}^T \E_{q_n \sim \hat{P}_n(t), \forall n} [\ell((q_n)_{n=1}^N, Y_t) - \ell(f(t), Y_t)] \\
  & \underset{\text{Lem. } \ref{lem:discretize_multi}}{\leq} 2\sup_{i \in [m], n \in [N], b \in [B]} \frac{1}{T} \sum_{t=1}^T \E_{q_n \sim \hat{P}_n(t), \forall n} [\ell_{\tau_n, \theta_i}(q_n, Y_t) - \ell_{\tau_n, \theta_i}(f_b^n(t), Y_t)] \\
  & \underset{Ieq. \eqref{ieq:hedge_w_update}}{\leq} \frac{2}{T} \sum_{t=1}^T \sum_{i=1}^{m} \sumn \sum_{b=1}^B w_{ib}^n(t) \, \E_{q_n \sim \hat{P}_n(t), \forall n}\left[ \ell_{\tau_n, \theta_i}(q_n, Y_t) - \ell_{\tau_n, \theta_i}(f_b^n(t), Y_t) \right] \\
  & \qquad \qquad + 2\left(c+\frac{1}{c}\right)\left(\sqrt{\frac{\log mNB}{T}}\right) \\
  & \underset{\text{Prop. } \ref{prop:Pn_is_minimax_soln}}{\leq} 2\left(c+\frac{1}{c}\right)\left(\sqrt{\frac{\log mNB}{T}}\right).
\end{align*}
\end{proof}

\subsection{Proof of Proposition \ref{prop:grid_consistency}} \label{pf:grid_consistency}
It suffices to treat an enlargement by one cell; the general case follows by iterating. Concretely, we compare two scenarios:
\begin{itemize}
  \item[\textup{S1:}] grid $\{0,\tfrac1m,\dots,\tfrac{m}{m}\}$, weights $w_{ib}^n$, $i \in [m]$;
  \item[\textup{S2:}] grid $\{0,\tfrac1m,\dots,\tfrac{m+1}{m}\}$, weights $\tilde w_{ib}^n$, $i \in [m+1]$,
\end{itemize}
and we mark S2 quantities with a tilde.

\smallskip
Consider un-normalized weights $u_{ib}^n(t)$ and their sum $U(t)$:
\[
u_{ib}^n(t)=\frac{1}{mNB}\exp\!\Big(\eta\!\sum_{s=1}^{t-1}\ell_{ib}^n(s)\Big),\quad U(t)=\!\!\sum_{i,b,n}\! u_{ib}^n(t),
\]
where $\ell_{ib}^n(s):=\mathbb{E}_{q_n\sim\hat P_n(s)}[\ell_{\tau_n,\theta_i}(q_n,Y_s)]-\ell_{\tau_n,\theta_i}(f_b^n(s),Y_s)$
is the update exponent in equation \eqref{eq:w_update}, and the analogous $\tilde u,\tilde U$ use initial mass $\tfrac1{(m+1)NB}$. Note that $u_{ib}^n(t) = w_{ib}^n(t)U(t)$ and that $U(t{+}1)=U(t)\,D(t)$, where $D(t)=\sum_{i,b,n}w_{ib}^n(t)e^{\eta\ell_{ib}^n(t)}$ is the normalizing denominator already computed in equation \eqref{eq:w_update}; thus maintaining $U(t)$ requires only one additional computation.

\paragraph{In-range equivalence.} We proceed by induction. We claim the following:
\[
\text{(a) } (\tilde{w}_{ib}^n(s))_{i,b,n} = c(s)({w}_{ib}^n(s))_{i\in[m], b,n} \quad \text{and} \quad \text{(b) } \tilde{\hat{P}}_n(s) = {\hat{P}}_n(s), \quad \forall s \leq t.
\]
First, (a) holds for $s=1$ since $w \equiv \frac{1}{mNB}$ and $\tilde{w} \equiv \frac{1}{(m+1)NB}$. Next, we show that if (a) holds at round $s$, then (b) holds at round $s$. Property 4 of Proposition~\ref{prop:Pn_def} states that
\begin{align}\label{eq:pn_def_property4}
\E_{q_n \sim \tilde{\hat P}_n(t)} \left[ \sum_{i=1}^{m+1} \sum_{b=1}^B \tilde{w}_{ib}^n \left( \indi{\theta_i < q_n} - \indi{\theta_i < f_b^n(t)} \right) \right] = \tilde{V}_{n-1}^*-\tilde{V}_n^*.
\end{align}
Since $f_b^n(t) \leq 1 < \theta_{m+1}$, we have $\indi{\theta_{m+1} < f_b^n(t)} = 0$, so index $m+1$ contributes nothing to $\tilde V_n(j)$ (eq. \eqref{eq:Vn_def}) for $j\le m$, and $\tilde V_n(m{+}1)-\tilde V_n(m)=\sum_{s=n+1}^N\sum_b\tilde w_{m+1, b}^s > 0$. This gives $\tilde{j}_n^* = j_n^*$ and $\tilde{V}_n^* = c(s)V_n^*$ at round $s$. The proof of Proposition~\ref{prop:Pn_def} implies $\tilde{k}_n = k_n \leq m$ and $\tilde{C}_n = C_n$ (eq. \eqref{def:cn}), and thus (b) holds at round $s$.
Lastly, we show that if (b) holds at round $s$ and $(f_b^n(s))_{b,n}, Y_s \leq 1$, then (a) holds at round $s+1$. This follows from the fact that, for all $i\in[m]$ and all $b, n$, $w_{ib}^n(s)$ and $\tilde{w}_{ib}^n(s)$ are multiplied by the same quantity, $\ell_{ib}^n(s)=\tilde{\ell}_{ib}^n(s)$, in the weight update (eq. \eqref{eq:w_update}). This closes the induction and proves in-range equivalence.

\paragraph{Reconstruction.}
Note that $\tilde{\ell}_{m+1,b}^n(s) = 0$ for $s \leq t$, and thus $\tilde{u}_{m+1,b}^n(s) = \frac{1}{(m+1)NB}$ for $s \leq t$. Also note that
\begin{align*}
\tilde{u}_{ib}^n(s) &= \frac{1}{(m+1)NB}\exp\left(\eta \sum_{s=1}^{t-1} \tilde{\ell}_{ib}^n(s)\right) \\
&= \frac{m}{m+1}\frac{1}{mNB}\exp\left(\eta \sum_{s=1}^{t-1} \ell_{ib}^n(s)\right) = \frac{m}{m+1}u_{ib}^n(s), \quad i \in [m].
\end{align*}
With $u_{ib}^n(s) = U(s)w_{ib}^n(s)$, we can recover $\tilde{w}(s)$ from $w(s)$.

\smallskip
\noindent\textbf{Expansion protocol.}
Within a round, the forecasts $f_b^n(t)$ (hence $\hat P(t)$) precede the target $Y_t$ (hence the weight update). If the first out-of-range value is a target $Y_t$, then $\hat P(t)$ was computed in range and carries no mass at $\tfrac{m+1}{m}$. We reconstruct $\tilde w(t)$ from $w(t)$ before updating. If the first out-of-range value is a forecast $f_b^n(t)$, the enlarged grid is already needed at prediction time, so we reconstruct $\tilde w(t)$ at the start of round $t$.

\subsection{Single quantile level case} \label{appssec:single_q}
In case $N=1$ with single quantile level $\tau$, the minimax program \eqref{eq:Pn_def_minimax} in Algorithm \ref{alg:omni_online} is replaced by 
\begin{align}\label{eq:minimax_single_q}
  \Phat(t) = \arg\min_{P \in \Delta_m} \max_{P^y \in \Delta_m} \sum_{i=1}^{m} \sum_{b=1}^B w_{ib} \,\E_{q\sim P, Y \sim P^y}[\ell_{\tau,\theta_i}(q,Y) - \ell_{\tau,\theta_i}(f_b(t),Y)].
\end{align}
The following proposition provides a closed-form expression for $\Phat(t)$ and show that the value of the minimax program is $0$.

\begin{proposition} \label{prop:minimax_sol_single_q}
Define $B_k := \sum_{i=1}^{k} \sum_{b=1}^B w_{ib} - \sum_{i=1}^{m} \sum_{b=1}^B w_{ib} \one(\theta_i < f_b(t))$ and let $k^* = \max\{k: B_k \leq 0\}$. Note that $B_{k^*} \leq 0$ and $B_{k^*+1}>0$. Then, 
\begin{align*}
P^* = \frac{B_{k^*+1}}{B_{k^*+1} + |B_{k^*}|}\delta_{\frac{k^*}{m}} + \frac{|B_{k^*}|}{B_{k^*+1} + |B_{k^*}|}\delta_{\frac{k^*+1}{m}} 
\quad\text{and}\quad
P^{y*} = \tau \delta_0 + (1-\tau) \delta_1 
\end{align*}
solves equation \eqref{eq:minimax_single_q}, and the value of the minimax program is $0$.
\end{proposition}

\begin{proof}
  Let $O(P,P^y) := \sum_{i=1}^{m} \sum_{b=1}^B w_{ib}\E_{q\sim P, Y \sim P^y}[\ell_{\tau,\theta_i}(q,Y) - \ell_{\tau,\theta_i}(f_b(t),Y)]$.
  It suffices to prove that $O(P^*, P^{y*}) = \max_{P^y} O(P^*, P^y) = \min_P O(P, P^{y*})$. Using the decomposition result from equation \eqref{eq:oppy_q}, we have 
\begin{align}\label{eq:oppy_q_single}
    & O(P, P^y) = \sum_{i=1}^{m} \sum_{b=1}^B w_{ib}
    \mathbb{E}_{q\sim P,\, Y \sim P^y} \Big[
    \one(Y \le \theta_i < q)
    - \tau \one(\theta_i < q)
    - \one(Y \le \theta_i < f_b(t))
    \Big] + C_\tau  ,
\end{align}
  where $C_\tau=\tau \sum_{i=1}^{m} \sum_{b=1}^B w_{ib} \one(\theta_i < f_b(t))$ which does not depend on $P$ or $P^y$.
  
  \smallskip
  We first prove $O(P^*, P^{y*}) = \max_{P^y} O(P^*, P^y)$. Let $P^y(\frac{j}{m})=p_{yj}$ for $j=0,...,m$. 
  Using the fact that $\one(\theta_i < k/m) = \one(i \leq k)$, $O(P^*, P^y)$ is rephrased as the following:
  \begin{align}
  &O(P^*, P^y) 
  = \E_{q\sim P^*} \left[\sum_{j=0}^m p_{yj} 
  \left(\sum_{i=j+1}^{m} \sum_{b=1}^B w_{ib} \left(\one(\theta_i < q) - \one(\theta_i < f_b(t))\right) \right) \right] + C(P^*) \nonumber\\
  &\qquad= \frac{B_{k^*+1}}{B_{k^*+1} + |B_{k^*}|} \left[\sum_{j=0}^m p_{yj} 
  \left(\sum_{i=j+1}^{k^*} \sum_{b=1}^B w_{ib} - \sum_{i=j+1}^{m} \sum_{b=1}^B w_{ib} \one(\theta_i < f_b(t)) \right) \right] \nonumber\\
  &\qquad\,\,\,+ \frac{|B_{k^*}|}{B_{k^*+1} + |B_{k^*}|} \left[\sum_{j=0}^m p_{yj} 
  \left(\sum_{i=j+1}^{k^*+1} \sum_{b=1}^B w_{ib} - \sum_{i=j+1}^{m} \sum_{b=1}^B w_{ib} \one(\theta_i < f_b(t)) \right) \right] + C(P^*) \label{eq:opstarpy_q},
  \end{align}
  where $C(P^*)$ is a value that does not depend on $P^y$ and we ignore the term hereafter.
  
  Now observe that $A_k^j := \sum_{i=j+1}^{k} \sum_{b=1}^B w_{ib} - \sum_{i=j+1}^{m} \sum_{b=1}^B w_{ib} \one(\theta_i < f_b(t))$ is non-increasing on $j\leq k$ and non-decreasing on $j \in [k, m)$ since $\one(\theta_i < f_b(t))$ is $0$ or $1$. Also note that $A_k^0 \geq A_k^{m-1}$ if $k=m$ and $A_k^m = 0 \geq A_k^{m-1}$ if $k \neq m$. Therefore, $A_k^j$ achieves maximum at $j=0$ or $j=m$:
  \begin{align*}
  & \max_{\sum_j{p_{yj}}=1}\sum_{j=0}^{m} p_{yj} \left(\sum_{i=j+1}^{k} \sum_{b=1}^B w_{ib} - \sum_{i=j+1}^{m} \sum_{b=1}^B w_{ib} \one(\theta_i < f_b(t))\right) \\
  &= \max_{p_{y0}+p_{ym}=1} p_{y0}\left(\sum_{i=1}^{k} \sum_{b=1}^B w_{ib} - \sum_{i=1}^{m} \sum_{b=1}^B w_{ib} \one(\theta_i < f_b(t))\right) + p_{ym}\cdot 0 \\
  &= \max_{p_{y0}+p_{ym}=1} p_{y0} \cdot B_k + p_{ym} \cdot 0.
  \end{align*}
  
  Putting this together with equation \eqref{eq:opstarpy_q} and the fact that $B_{k^*}\leq 0$ gives
  \begin{align*}
  &\max_{P^y}O(P^*, P^y) \\
  &= \max_{p_{y0}+p_{ym}=1} p_{y0} \left(\frac{B_{k^*+1}}{B_{k^*+1} + |B_{k^*}|} \cdot B_{k*} + \frac{|B_{k^*}|}{B_{k^*+1} + |B_{k^*}|} \cdot B_{k^*+1}\right) + p_{ym}\cdot 0 + C(P^*),
  \end{align*}
  which is achieved by $P^{y*}$ since $P^{y*}$ only puts mass on $j=0, m$. This proves that $O(P^*, P^{y*}) = \max_{P^y} O(P^*, P^y)$.
  
  \smallskip
  Next, we prove $O(P^*, P^{y*}) = \min_P O(P, P^{y*})$. The decomposition \eqref{eq:oppy_q_single} of $O(P, P^y)$ reduces to
  \begin{align*}
O(P, P^{y*}) &= \tau \sum_{i=1}^{m} \sum_{b=1}^B w_{ib}\E_{q\sim P}[(1-\tau)\one(\theta_i < q) - \one(\theta_i < f_b(t))] + \\
  & \qquad (1-\tau) \sum_{i=1}^{m} \sum_{b=1}^B w_{ib}\E_{q\sim P}[-\tau\one(\theta_i < q)] + C_\tau \\
  &= -\tau \sum_{i=1}^{m} \sum_{b=1}^B w_{ib}\one(\theta_i < f_b(t)) + C = 0,
  \end{align*}
  which does not depend on $P$. Thus, $O(P^*, P^{y*}) = \min_P O(P, P^{y*})$.
\end{proof}

\section{Additional algorithmic details} \label{appsec:fast}

In this section, we provide additional algorithmic details.
To start, Algorithm \ref{alg:omni_online_full} provides the full description of our main omniprediction algorithm, Algorithm \ref{alg:omni_online}. 
Appendix \ref{appsubsec:multi_q_fast} discusses an efficient subroutine for computing $V_n^*$ and $j_n^*$, defined in equations \eqref{eq:Vn_def_multi_q}; and Appendix \ref{appsubsec:wql_fast} discusses an efficient subroutine for computing $\theta_n^*$, defined in equations \eqref{eq:Vn_def_wql}.

\begin{algorithm}[htbp]
  \caption{Full description of Algorithm \ref{alg:omni_online}} \label{alg:omni_online_full}
  \begin{algorithmic}[1]
  \State \textbf{Data:} Data stream $(Y_t)_{t=1}^T$, quantile levels $\tau_1 < \tau_2 < ... < \tau_N$, Grid size $m\in \mathbb{N}$, Learning rate $\eta> 0$, Base forecasters: $\{f_b\}_{b=1}^B$ where $f_b$ outputs $f_b^n(t)$ as $\tau_n$-level quantile estimate at time $t$;
  \State \textbf{Initialize:} $w_{ib}^n(1) = \frac{1}{mNB}$, $\forall i \in [m], n \in [N], b \in [B]$;
  \For{$t = 1, \ldots, T$}
    \State \textbf{1) Make prediction at time $t$}:
      \State Define $V_0^* = V_N^* = 0$, $j_0^* = 0$, $j_N^* = m$;
      \State Define $w_i^n = (w_{ib}^n(t))_{b=1}^B$, $\one_i^n = (\indi{\theta_i < f_b^n(t)})_{b=1}^B$ and $\one_B = (1,..,1) \in \R^B$ for $\forall i, n$.
      \For{$n=1, \ldots, N-1$}
      \State Define $V_n(j)$ as 
        \(
        V_n(j) = \sum_{i=j+1}^{m} \sum_{s=1}^{n} w_i^s \cdot \one_i^s + \sum_{i=1}^{j} \sum_{s=n+1}^{N}  w_i^s \cdot (\one_B - \one_i^s)
        \)
      \State Compute $V_n^*$ and $j_n^*$:
        \begin{align} \label{eq:Vn_def_multi_q}
          V_n^* = \min_{j \in [m]} V_n(j), \qquad
          j_n^* = \min\left\{j \mid j\in [j_{n-1}^*, m],  V_n(j) = V_n^* \right\}
        \end{align}
      \EndFor \vspace{-0.2cm}\\
      \For{$n=1, \ldots, N$}
        \If{$V_n(j_{n-1}^*) = V_n(j_n^*)$}
          \State $\hat{P}_n(t) = \delta_{j_{n-1}^*/m}$;
        \Else
          \State Define $B_k := \sum_{i=1}^{k} w_i^n \cdot \one_B - \sum_{i=1}^{m}w_i^n \cdot \one_i^n$
          \State Find $j_{n-1}^* \leq k_n < j_n^*$ s.t. $B_{k_n} \!<\! V_{n-1}^* \!-\! V_n^* \!\leq\! B_{k_n+1}$ (Existence proved in Prop. \ref{prop:Pn_def})
          \State $\hat{P}_n(t) = (1\!-\!C_n)\delta_{k_n/m} + C_n\delta_{((k_n\!+\!1)/m)}$, where $C_n = \left( V_{n\!-\!1}^* \!-\! V_n^* \!-\! B_{k_n} \right) / (w_{k_n}^n \!\cdot\! \one_B)$
        \EndIf
      \EndFor
      \State \textbf{Output} $\Pbhat(t) = (\hat{P}_n(t))_{n=1}^N$;  \vspace{-0.1cm}\\
      \State \textbf{2) Update the weights for each $S_{n,i}$ and $f_b$}:
          \[
          w_{ib}^n(t+1) 
          = \frac{w_{ib}^n(t) \exp\bigl(\eta \big(\mathbb{E}_{q_n \sim \hat{P}_n(t)}[\ell_{\tau_n,\theta_i}(q_n, Y_t)] 
          - \ell_{\tau_n,\theta_i}({f}_b^n(t), Y_t)\big)\bigr)}
          {\sum_{i, n, b} w_{ib}^n(t) \exp\bigl(\eta \big(\mathbb{E}_{q_n \sim \hat{P}_n(t)}[\ell_{\tau_n,\theta_i}(q_n, Y_t)] - \ell_{\tau_n,\theta_i}({f}_b^n(t), Y_t)\big)\bigr)}, \forall i, n, b
          \]
  \EndFor
  \end{algorithmic}
\end{algorithm}

\subsection{Efficient subroutine for Algorithm \ref{alg:omni_online_full}} \label{appsubsec:multi_q_fast}
This subsection introduces Algorithm \ref{alg:compute_vn}, an efficient subroutine to compute $V_n^*$ and $j_n^*$ for each $n \in [N-1]$ in $O(mNB)$ time. Recall the definition of $V_n^*$ from Lemma \ref{lem:jn_sort}:
\[
  V_n^* := \min_{j \in [m]_0} \left( \sum_{i=j+1}^{m} \sum_{s=1}^{n} \sum_{b=1}^B w_{ib}^s \indi{\theta_i < f_b^s(t)} + \sum_{i=1}^{j} \sum_{s=n+1}^{N}  \sum_{b=1}^B w_{ib}^s (1 - \indi{\theta_i < f_b^s(t)})\right) .
\]

\begin{algorithm}[h]
  \caption{Efficient computation of $V_n$ and its minimizer $j_n^*$} \label{alg:compute_vn}
  \begin{algorithmic}[1]
  \State \textbf{Input:} Weights $(w_{ib}^s)_{s \in [N],\, i \in [m],\, b \in [B]}$, indicators $(\indi{\theta_i < f_b^s(t)})_{s \in [N],\, i \in [m],\, b \in [B]}$
  \State \textbf{Output:} Values $(V_n)_{n=1}^{N-1}$ and minimizers $(j_n^*)_{n=1}^{N-1}$
  
  \State \textbf{1) Compute aggregated weights and indicators:}
  \[
  w_i^s = \sum_{b=1}^B w_{ib}^s,
  \qquad
  I_i^s = \sum_{b=1}^B w_{ib}^s \indi{\theta_i < f_b^s(t)}, \qquad \forall i \in [m], s \in [N].
  \]
  
  \State \textbf{2) Compute cumulative sums:}
  \begin{align*}
    \text{FrontSum}[n, i]= \sum_{s=1}^{n} I_i^s, \quad
    \text{BackSum}[n, i] = \sum_{s=n+1}^{N} (w_i^s - I_i^s), \quad \forall i \in [m], n \in [N-1].
  \end{align*}
  
  \State \textbf{3) Find $j_n^*$ for each $n \in [N-1]$:}
  \State Define $j_0^* = 0$. Initialize $\mathrm{cost}[j] = 0, \forall j \in [m]_0$.
  \For{$n = 1,\dots,N-1$}
      \State Compute the cost at $j=j_{n-1}^*$:
      \[
      \mathrm{cost}[j_{n-1}^*] = \sum_{i=j_{n-1}^*+1}^m \text{FrontSum}[n,i]
      \;+\;
      \sum_{i=1}^{j_{n-1}^*} \text{BackSum}[n,i], \qquad
      V_n \gets \mathrm{cost}[j_{n-1}^*]
      \]
      \For{$j = j_{n-1}^*\!+\!1, \dots, m$}
          \State $\mathrm{cost}[j] = \mathrm{cost}[j\!-\!1] - \text{FrontSum}[n, j\!-\!1] + \text{BackSum}[n, j\!-\!1]$
          \State $V_n \gets \min(V_n, \, \mathrm{cost}[j])$
      \EndFor
      \State $j_n^* = \argmin_{j_{n-1}^* \leq j \leq m} \mathrm{cost}[j]$
  \EndFor
  
  \State \textbf{Return:} $(V_n)_{n=1}^{N-1}$ and $(j_n^*)_{n=1}^{N-1}$
  \end{algorithmic}
\end{algorithm}

\subsection{Efficient subroutine for Algorithm \ref{alg:omni-twoplayer_online_wql}} \label{appsubsec:wql_fast}
This subsection introduces Algorithm \ref{alg:compute_theta_n_wql} that computes $\{\theta_n^*\}_{n=0}^N$ in $O(K\log(K))$ time, where $K$ is number of unique values in $\{f_b^n(t)\}_{b,n}$. With this algorithm, now the computational bottleneck lies in computing $V_n^* = V_n(\theta_n^*)$ for each $n \in [N-1]$, which takes $O(N^2B)$ in total. To reduce the time complexity even further, we introduce Algorithm \ref{alg:compute_vn_diff_wql} that computes $V_{n-1}^* - V_n^*$ for all $n \in [N-1]$, which is all we need to to compute the minimax solution $q^G_n(t)$ (See Proposition \ref{prop:theta_n_pn_sorted_wql}), in $O(NB)$ time. \smallskip

We begin with introducing an efficient algorithm to compute $\{\theta_n^*\}_{n=0}^N$ in $O(K\log(K))$ time. Recall the definition of $V_n^*$ in Proposition \ref{prop:theta_n_pn_sorted_wql}:
\[
V_n(\theta) := \sum_{s=1}^{n} \sum_{b=1}^B w_b^s \max(f_b^s(t)-\theta, 0) + \sum_{s=n+1}^{N}\sum_{b=1}^B  w_b^s \max(\theta - f_b^s(t), 0), 
\qquad V_n^* := \min_{\theta \in \R} V_n(\theta).
\]
The key point in efficiently computing $V_n^*$ is that $V_n(\theta)$ is a sum of piecewise linear functions (hinge functions) in $\theta$. This implies that the minimizer $\theta_n^*$ is the point where the slope of $V_n(\theta)$ changes from negative to positive. Another key observation is that the slope of $V_n(\theta)$ only changes at $\{f_b^s(t)\}_{b,s}$ and it increases by $w_b^s$. More precisely, the change of slope at $f_b^s(t)$ is $\sum_{\{(x,z): f_x^z(t) = f_b^s(t)\}} w_x^z$. It is important to note that the amount of change at each $f_b^s(t)$ is same for all $V_n(\theta)$, $n \in [N-1]$.

Then, we observe that the slope of $V_n(\theta)$ at $\theta < \min_{b,n} f_b^n(t)$ is $- \sum_{s=1}^{n} \sum_{b=1}^B w_b^s$. This implies that finding $\theta_n^*$ reduces to find the smallest $f_b^s(t)$ such that cumulative slope increase is larger than or equal to $\sum_{s=1}^{n} \sum_{b=1}^B w_b^s$. Formally, $\theta_n^*$ equals to the smallest $f_b^s(t)$ such that $\sum_{\{(x,z): f_x^z(t) \leq f_b^s(t)\}} w_x^z \geq \sum_{s=1}^{n} \sum_{b=1}^B w_b^s$.

\begin{algorithm}[h]
  \caption{Efficient computation of $\{\theta_n^*\}_{n=0}^N$}
  \label{alg:compute_theta_n_wql}
  \begin{algorithmic}[1]
  \State \textbf{Input:} Weights $(w_b^n)_{n \in [N],\, b \in [B]}$, forecasts $\{f_b^n(t)\}_{n \in [N],\, b \in [B]}$

  \State \textbf{1) Compute cumulative row sums:}
  \State Define $W[n] = \sum_{s=1}^{n} \sum_{b=1}^B w_b^s, \quad n \in [N]$.

  \State \textbf{2) Compute the slope increases at each $f_b^s(t)$:}
  \State Flatten $\{(w_b^n, f_b^n(t))\}_{b,n}$ into pairs $\{(\tilde{w}_k,\tilde{f}_k)\}_{k=1}^{NB}$ and sort them by $\tilde{f}_k$ in nondecreasing order.
  \State Group equal forecast values and let the distinct sorted values be $\{g_k\}_{k=1}^K$ with $g_1 \!<\! g_2 \!<\! \cdots \!<\! g_K$.
  \State For each $k \in [K]$, compute the sum of corresponding weights: $v_k = \sum_{\{(b,n): f_b^n(t) = g_k\}} w_b^n$.
  \State Compute the cumulative sum of $v_k$: $Vcum[k] = \sum_{\ell=1}^{k} v_\ell$.
  
  \State \textbf{3) Find $\theta_n^*$ for each $n \in [N-1]$:}
  \For{$n = 1, \ldots, N-1$}
    \State Binary search on $Vcum$ (sorted b/c $v_k > 0$) to find smallest $k_n^*$ such that $Vcum[k_n^*] \geq W[n]$.
    \State $\theta_n^* \gets g_{k_n^*}$.
  \EndFor
  \end{algorithmic}
\end{algorithm}

\smallskip
Next, we introduce an algorithm that computes $V_{n-1}^* - V_n^*$ for all $n \in [N-1]$ in $O(NB)$ time. Note that naive computation of $V_n^*$ using $\theta_n*$ takes $O(NB)$ time and thus $O(N^2B)$ in total. Before introducing the algorithm, we observe what $V_{n-1}^* - V_n^*$ is:
\begin{align*} 
V_{n-1}(\theta_{n-1}^*) - V_n(\theta_n^*) 
& = \sum_{s=1}^{n-1} \sum_{b=1}^B w_b^s \left(\max(f_b^s(t)-\theta_{n-1}^*, 0) - \max(f_b^s(t)-\theta_n^*, 0)\right) \\
& \quad + \sum_{s=n+1}^{N}\sum_{b=1}^B  w_b^s \left(\max(\theta_{n-1}^* - f_b^s(t), 0) - \max(\theta_n^* - f_b^s(t), 0)\right) \\
& \quad + \sum_{b=1}^B w_b^n \left( \max(\theta_{n-1}^* - f_b^n(t), 0) - \max(f_b^n(t) - \theta_n^*, 0)\right) .
\end{align*}
Each component results in different values depending on whether $f_b^n(t) < \theta_{n-1}^*$ or $\theta_{n-1}^* \leq f_b^n(t) < \theta_n^*$ or $\theta_n^* \leq f_b^n(t)$. Fortunately, some calculations on the quantity above (divide cases, etc.) gives the following simple formula:
\begin{align*}
& V_{n-1}(\theta_{n-1}^*) - V_n(\theta_n^*) \\
& = \sum_{s=1}^N \sum_{b=1}^B w_b^n 
\bigl( \theta_{n-1}^* \, \indi{f_b^s(t) < \theta_{n-1}^*}  
+ f_b^s(t) \,\indi{\theta_{n-1}^* \leq f_b^s(t) < \theta_n^*}  
+ \theta_n^*  \, \indi{\theta_n^* \leq f_b^s(t)} \bigr)  \\
& \quad - \theta_{n-1}^*\left(\sum_{s=1}^{n-1} \sum_{b=1}^B w_b^s\right) - \theta_n^* \left(\sum_{s=n+1}^{N}\sum_{b=1}^B w_b^s\right) - \sum_{b=1}^B w_b^n f_b^n(t).
\end{align*}
Observe that most quantities can be precomputed in $O(NB)$ time. Algorithm \ref{alg:compute_vn_diff_wql} shows the details.

\begin{algorithm}[h]
  \caption{Efficient computation of $V_{n-1}^* - V_n^*$ for all $n \in [N-1]$}
  \label{alg:compute_vn_diff_wql}
  \begin{algorithmic}[1]
  \State \textbf{Input:} Weights $(w_b^n)_{b,n}$, forecasts $\{f_b^n(t)\}_{b,n}$, $(g_k)_{k=1}^K$, $\{\theta_n^* = g_{k_n^*}\}_{n=0}^N$ from Algorithm \ref{alg:compute_theta_n_wql}
  \State \textbf{Output:} $(V_{n-1}^* - V_n^*)_{n=1}^{N-1}$
  \State \textbf{1) Compute cumulative row sums:}
  \State Define $W[n] = \sum_{s=1}^{n} \sum_{b=1}^B w_b^s, \quad n \in [N]$, $W[0] = 0$.
  \State Define $WB[n] = \sum_{s=n+1}^N \sum_{b=1}^B w_b^s, \quad n \in [N-1]$, $WB[N] = 0$. \smallskip
  \State Define $V[k] = \sum_{\{(b,n): f_b^n(t) = g_k\}} w_b^n, \quad k \in [K]$.
  \State Define $VF[k] = \sum_{\{(b,n): f_b^n(t) = g_k\}} w_b^n * f_b^n(t), \quad k \in [K]$.
  \State Define $Vcum[k] = \sum_{\ell=1}^{k} V[k], \quad k \in [K]$, $Vcum[0] = 0$.
  \State Define $VFcum[k] = \sum_{\ell=1}^{k} VF[k], \quad k \in [K]$, $VFcum[0] = 0$. \smallskip

  \State \textbf{2) Compute $V_{n-1}^* - V_n^*$ for each $n \in [N-1]$:}
  \For{$n = 1, \ldots, N-1$}
    \begin{align*}
      V_{n-1}^* - V_n^* 
      & = \theta_{n-1}^* (Vcum[\max(k_{n-1}^*,0)]) + (VFcum[k_n^*-1] - VFcum[\max(k_{n-1}^*,0)]) \\
      & \quad + \theta_n^* (1- Vcum[k_n^*-1]) - \theta_{n-1}^* (W[n-1]) - \theta_n^* (WB[n]) - \sum_{b=1}^B w_b^n f_b^n(t).  
    \end{align*}
    
  \EndFor
  \end{algorithmic}
\end{algorithm}

\subsection{Runtime analysis} 
\label{appsubsec:runtime_analysis}
\begin{figure}[h]
  \centering
  \includegraphics[width=0.9\textwidth]{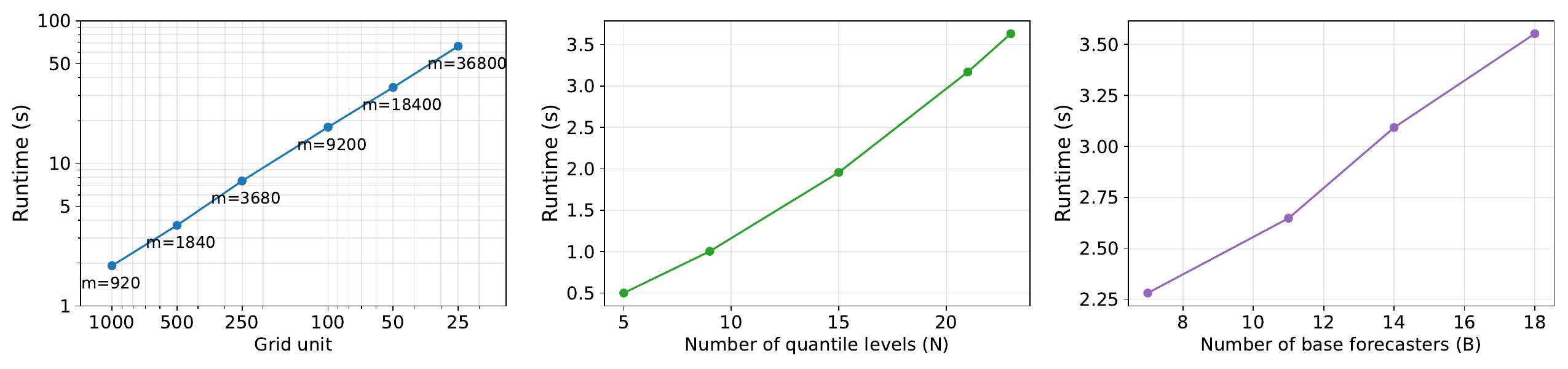}
  \caption{Runtime of Algorithm \ref{alg:omni_online} with $c_O=1$ on U.S.-aggregate, 4-week-ahead prediction data versus grid size $m$ (\textbf{Left}), number of quantile levels $N$ (\textbf{Middle}), and number of base forecasters $B$ (\textbf{Right}). Middle and right plots use grid size $500$.}
  \label{fig:runtime_analysis}
\end{figure}

{We empirically verify the $O(mNB)$ per-step complexity established in Proposition~\ref{prop:omni_subroutine}. Figure~\ref{fig:runtime_analysis} shows the runtime of Algorithm~\ref{alg:omni_online} while varying one of grid size $m$, number of quantile levels $N$, or number of base forecasters $B$ and holding the other two fixed. In each case the runtime grows linearly, consistent with the result of Proposition \ref{prop:omni_subroutine}.}

\section{Omniprediction under weighted quantile losses} \label{appsec:multi_wql}
In this section, we develop omnipredictors for the restricted subset of proper loss functions that can be expressed as weighted combinations of pinball losses. Although this is a much smaller set than the set of all proper losses for multiple quantile levels, which we discussed in Section \ref{sec:theory}, it has the advantage of not requiring any discretization of the target outcome or base forecasters' predictions. 

We introduce two algorithms for omniprediction for weighted quantile losses. In Appendix \ref{subsec:hedge_wql}, we discuss the first algorithm which runs the Hedge algorithm on each quantile level then combines quantile estimates from each level. In Appendix \ref{subsec:twoplayer_wql}, we introduce a two-player game based algorithm that has a similar form to the omniprediction algorithm for multiple quantile levels given in Algorithm~\ref{alg:omni_online}. Finally, in Appendix \ref{subsec:comparison_wql}, we compare the two algorithm in terms of time complexity, quantile crossing, and omniprediction error bound.

\paragraph{Notation.}
Define the pinball loss at quantile level $\tau$ as $\rho_\tau(x) := \tau x\indi{x \geq 0} +(1-\tau)|x|\indi{x < 0} = (\tau - \indi{x < 0})x$. The weighted quantile loss at quantile levels $\{\tau_n\}_{n=1}^N$ is defined as 
\[
\ell\left((x_n)_{n=1}^N, y\right) := \sum_{n=1}^N w_n \rho_{\tau_n}(y - x_n),
\]
where $w_n$ is the weight for the $\tau_n$-level quantile prediction and $\sum_{n=1}^N w_n = 1$. We denote $\cl_{wql}$ as the set of all weighted quantile loss functions and use $\rho_n(x, y)$ to denote $\rho_{\tau_n}(y-x)$ for $n \in [N]$. We define $f_\text{min}(t) = \min_{n \in [N], b \in [B]} f_b^n(t)$ and $f_\text{max}(t) = \max_{n \in [N], b \in [B]} f_b^n(t)$ for each time $t$. \smallskip

We aim to come up with $\mathbf{q}(t) = (q_n(t))_{n=1}^N$ that minimizes the following omniprediction objective:
\[
\sup_{\ell \in \cl_{wql},\, b \in [B]} \frac{1}{T} \sum_{t=1}^T [\ell(\mathbf{q}(t), Y_t) - \ell(f_b(t), Y_t)],
\]
where $f_b(t) := (f_b^n(t))_{n=1}^N$.

\subsection{Hedge algorithm on each quantile level} \label{subsec:hedge_wql}
In this subsection, we introduce the first algorithm that essentially runs Hedge algorithm on each quantile level then provide the omniprediction error bound for this algorithm. We denote the output from Algorithm \ref{alg:wql_hedge} by $\mathbf{q}^H(t) := (q_n^H(t))_{n=1}^N$, where the superscript $H$ denotes it is from the Hedge algorithm.
 
\begin{algorithm}[htbp]
\caption{Hedge algorithm for multiple quantile levels in parallel} \label{alg:wql_hedge}
\begin{algorithmic}[1]
\State \textbf{Data:} Data stream $(Y_t)_{t=1}^T$, Quantile levels $\{\tau_n\}_{n=1}^N$, Learning rate $\eta> 0$, Base forecasters: $\{f_b\}_{b=1}^B$ where $f_b$ outputs $f_b^n(t)$ as $\tau_n$-level quantile estimate at time $t$;
\State \textbf{Initialize:} $w_{b}^n(1) = \frac{1}{B}$, $\forall b \in [B], n \in [N]$;
\For{$t = 1, \ldots, T$}
    \State \textbf{Output:} Output $\mathbf{q}^H(t) := (q_n^H(t))_{n=1}^N$ where $q_n^H(t) = \sum_{b=1}^B w_{b}^n(t) f_b^n(t)$;
    \State Update weights $w_{b}^n$ for $b \in [B], n \in [N]$:
    \[
    w_{b}^n(t+1) 
    = \frac{w_{b}^n(t) \exp\left(-\eta \rho_n(f_b^n(t), Y_t)\right)}{\sum_{b=1}^B w_{b}^n(t) \exp\left(-\eta \rho_n(f_b^n(t), Y_t)\right)}
    \]
\EndFor
\end{algorithmic}
\end{algorithm}

This algorithms yields the following guarantee for each quantile level $\tau_n$:
\begin{align*}
\frac{1}{T}\sum_{t=1}^T \rho_n(q_n^H(t), Y_t) 
& \leq \frac{1}{T}\sum_{t=1}^T w_b^n(t) \rho_n(f_b^n(t), Y_t)  \\
& \leq \min_b \frac{1}{T}\sum_{t=1}^T \rho_n(f_b^n(t), Y_t) + \frac{\eta}{T} \sum_{t=1}^T \sum_{b=1}^B w_{b}^n(t) \rho_n(f_b^n(t), Y_t)^2 + \frac{\log B}{\eta T},
\end{align*}
where the first inequality is due to convexity of $\rho_\tau$ and the second inequality is due to the regret bound of the Hedge algorithm (Theorem \ref{appthm:hedge_regret}). This gives the omniprediction error bound for Algorithm \ref{alg:wql_hedge} as follows:
\begin{align}
    & \sup_{\ell \in \cl_{wql},\, b \in [B]} \frac{1}{T} \sum_{t=1}^T [\ell(\mathbf{q}^H(t), Y_t) - \ell(f_b(t), Y_t)] \nonumber\\
    & \qquad = \max_{n \in [N], b \in [B]} \frac{1}{T} \sum_{t=1}^T [\rho_n(q_n^H(t), Y_t) - \rho_n(f_b^n(t), Y_t)] \nonumber \\
    & \qquad \leq \eta \max_{n\in[N]} \left( \frac{1}{T} \sum_{t=1}^T \sum_{b=1}^B w_{b}^n(t) \rho_n(f_b^n(t), Y_t)^2\right)  + \frac{\log B}{\eta T} \nonumber \\
    & \qquad \leq \eta \max_{n\in[N]} \left( \frac{1}{T} \sum_{t=1}^T \max_{b} \rho_n(f_b^n(t), Y_t)^2  \right) + \frac{\log B}{\eta T} .
    \label{eq:hedge_wql_error_bound}
\end{align}

\subsection{Two-player game based algorithm} \label{subsec:twoplayer_wql}
In this subsection, we introduce a second algorithm that is based on two-player game that is similar to Algorithm \ref{alg:omni_online} in Section \ref{sec:theory}, and then we provide the omniprediction error bound for this algorithm. In the process, we prove that the output from this algorithm does not exhibit quantile crossing. We denote the output from Algorithm \ref{alg:omni-twoplayer_online_wql} by $\mathbf{q}^G(t) := (q^G_n(t))_{n=1}^N$, where the superscript $G$ denotes it is from the two-player Game based algorithm.

\begin{algorithm}[htbp]
    \caption{Two-player game based omniprediction for weighted quantile losses} \label{alg:omni-twoplayer_online_wql}
    \begin{algorithmic}[1]
    \State \textbf{Data:} Data stream $(Y_t)_{t=1}^T$, quantile levels $\tau_1 < \tau_2 < ... < \tau_N$, Learning rate $\eta> 0$, Base forecasters: $\{f_b\}_{b=1}^B$ where $f_b$ outputs $f_b^n(t)$ as $\tau_n$-level quantile estimate at time $t$;
    \State \textbf{Initialize:} $w_{b}^n(1) = \frac{1}{NB}$, $\forall n \in [N], b \in [B]$;
    \For{$t = 1, \ldots, T$}
      \State \textbf{1) Make prediction at time $t$}:
        \State Define $V_0^* = V_N^* = 0$, $\theta_0^* = f_\text{min}(t)$, $\theta_N^* = f_\text{max}(t)$;
        \For{$n=1, \ldots, N-1$}
        \State Define $V_n(\theta)$ as:
          \[
          V_n(\theta) = \sum_{s=1}^{n} \sum_{b=1}^B w_b^s \max(f_b^s(t)-\theta, 0) + \sum_{s=n+1}^{N} \sum_{b=1}^B  w_b^s \max(\theta - f_b^s(t), 0)
          \]
        \State Compute $V_n^*$ and $\theta_n^*$:
          \begin{align} \label{eq:Vn_def_wql}
            V_n^* := \min_{\theta \in [f_\text{min}(t), f_\text{max}(t)]} V_n(\theta), \qquad
            \theta_n^* = \min\left\{\theta \mid \theta\in [\theta_{n-1}^*, f_\text{max}(t)],  V_n(\theta) = V_n^* \right\}
          \end{align}
        \EndFor \vspace{-0.2cm}\\
        \State \textbf{Output} $\mathbf{q}^G(t) := (q^G_n(t))_{n=1}^N$ where 
        \[
        q^G_n(t) = \frac{1}{\sum_{b=1}^B w_b^n(t)} \left( \sum_{b=1}^B w_b^n(t) f_b^n(t) + V_{n-1}^* - V_n^*  \right);
        \] 
        \State \textbf{2) Update the weights for each $\rho_n$ and $f_b$}:
            \[ 
            w_{b}^n(t+1) 
            = \frac{w_{b}^n(t) \exp\bigl(\eta [\rho_n(q^G_n(t), Y_t) - \rho_{n}(f_b^n(t), Y_t)]\bigr)}
            {\sum_{n=1}^N \sum_{b=1}^B w_{b}^n(t) \exp\bigl(\eta [\rho_n(q^G_n(t), Y_t) - \rho_{n}(f_b^n(t), Y_t)]\bigr)}, \quad \forall n \in [N], b \in [B];
            \]
    \EndFor
    \end{algorithmic}
\end{algorithm}

We show that $(q^G_n(t))_{n=1}^N$ does not exhibit quantile crossing.
\begin{proposition}\label{prop:theta_n_pn_sorted_wql}
Define a function $V_n(\theta)$ and a quantity $V_n^*$ for $n \in [N-1]$ as
\begin{align*}
    V_n(\theta) &:= \sum_{s=1}^{n} \sum_{b=1}^B w_b^s \max(f_b^s(t)-\theta, 0) + \sum_{s=n+1}^{N}\sum_{b=1}^B  w_b^s \max(\theta - f_b^s(t), 0) \geq 0 \\
    V_n^* &:= \min_{\theta \in \R} V_n(\theta)  ,    
\end{align*}
and $V_0^* = V_N^* = 0$. Let $J_n^*$ be the set of $\theta$s that minimizes the $V_n(\theta)$. Then, there exists a sorted sequence $f_\text{min}(t) \leq \theta_0^* \leq \theta_1^* \leq ... \leq \theta_{N-1}^* \leq \theta_N^* = f_\text{max}(t)$ such that $\theta_n^* \in J_n^*$ for $n \in [N-1]$.
\smallskip

Additionally, $q_n$ defined as 
\begin{align} \label{eq:pn_def_wql}
q_n = \frac{1}{\sum_{b=1}^B w_b^n(t)} \left( \sum_{b=1}^B w_b^n(t) f_b^n(t) + V_{n-1}^* - V_n^*  \right) 
\end{align}
satisfies $\theta_{n-1}^* \leq q_n \leq \theta_n^*$ for $n \in [N-1]$, and thus $f_\text{min}(t) \leq q_1 \leq q_2 ... \leq q_N \leq f_\text{max}(t)$ holds.
\end{proposition}

\begin{proof}
Since $V_n(\theta)$ shares the same structure with $V_n(j)$ defined in equation \eqref{eq:Vn_def}, existence of a sorted sequence follows the same proof in Lemma \ref{lem:jn_sort}. To prove that such $\theta_n^*$ takes values in $[f_\text{min}(t), f_\text{max}(t)]$ for all $n \in [N-1]$, it suffices to show that $J_n^* \cap [f_\text{min}(t), f_\text{max}(t)] \neq \emptyset$ for all $n \in [N-1]$. On $\theta \leq f_\text{min}(t)$, the second sum in $V_n(\theta)$ is simply zero and thus $V_n(\theta)$ is non-increasing. On $\theta \geq f_\text{max}(t)$, the first sum in $V_n(\theta)$ is zero and thus $V_n(\theta)$ is non-decreasing. Therefore, $J_n^*$ cannot only contain values that are not in $[f_\text{min}(t), f_\text{max}(t)]$. \smallskip

Next, we prove that $\theta_{n-1}^* \leq q_n \leq \theta_n^*$ for $n \in [N-1]$. Observe the identity:
\[
V_{n-1}(\theta) - V_{n}(\theta) = \sum_{b=1}^B w_b^n(t) (\max(\theta - f_b^n(t), 0) - \max(f_b^n(t)-\theta, 0)) = \sum_{b=1}^B w_b^n(t) (\theta-f_b^n(t)),
\] 
which implies that 
\begin{align*}
q_n & = \frac{1}{\sum_{b=1}^B w_b^n(t)}\left( \sum_{b=1}^B w_b^n(t) f_b^n(t) + V_{n-1}^* - V_{n-1}(\theta_n^*) + \sum_{b=1}^B w_b^n(t) (\theta_n^* - f_b^n(t)) \right) \leq \theta_n^*,
\end{align*}
since $V_{n-1}^* - V_{n-1}(\theta_n^*) \leq 0$ by the definition of $V_n^*$. Similarly, 
\begin{align*}
    q_n & = \frac{1}{\sum_{b=1}^B w_b^n(t)}\left( \sum_{b=1}^B w_b^n(t) f_b^n(t) + V_{n-1}(\theta_{n-1}^*) - V_n^* + \sum_{b=1}^B w_b^n(t) (\theta_{n-1}^* - f_b^n(t)) \right) \geq \theta_{n-1}^*,
\end{align*}
which completes the proof.
\end{proof}

Next, we show that $(q_n)_{n=1}^N$ are a solution to the minimax program . 
\begin{theorem}\label{thm:minimax_wql}
    Given quantile levels $\tau_1 < \tau_2 < ... < \tau_N$ and quantile forecasts ${f}_b^n(t)$ for $b \in [B]$, $n \in [N]$, define $V_n^*$ and $\theta_n^*$ for $n \in [N]\cap\{0\}$ as in Proposition \ref{prop:theta_n_pn_sorted_wql} and define $p_n^*$, $n \in [N]$ as in equation \eqref{eq:pn_def_wql}. Then, $P_n^* \equiv q_n$ for $n \in [N]$ with 
    \[
    P^{y*} = \tau_1 \delta_{f_\text{min}(t)} + (1-\tau_N)\delta_{f_\text{max}(t)} + \sum_{n=1}^{N-1}(\tau_{n+1}-\tau_n)\delta_{\theta_n^*}
    \]
    is a solution of the minimax program 
    \[
    \min_{P_1, P_2, ..., P_N \in \cp(\R)} \max_{P^y \in \cp(\R)} \sum_{n=1}^N \sum_{b=1}^B w_b^n(t) \E_{q_n \sim P_n, Y \sim P^y} [\rho_n(q_n, Y) - \rho_n(f_b^n(t), Y)]
    \]
    and the value of the minimax program is 
    \(
    -\sum_{n=1}^{N-1}(\tau_{n+1}-\tau_n) V_n^* \leq 0.
    \)
\end{theorem}

\begin{proof}
  Similar to the proof of Proposition \ref{prop:Pn_is_minimax_soln}, define $O\left((P_n)_{n=1}^N, P^y\right)$ as 
  \[
  O\left((P_n)_{n=1}^N, P^y\right) = \sum_{n=1}^N \sum_{b=1}^B w_b^n(t) \E_{q_n \sim P_n, Y \sim P^y} [\rho_n(q_n, Y) - \rho_n(f_b^n(t), Y)].
  \]
  It suffices to prove
  \[
  O\left((P_n^*)_{n=1}^N, P^{y*}\right)= \max_{P^y} O\left((P_n^*)_{n=1}^N, P^y\right) = \min_{P_n \in \cp(\R), n \in [N]} O\left((P_n)_{n=1}^N, P^{y*}\right).
  \] 
  Observe that $\rho_\tau(x,y) = \tau(y-x) + \max(x-y,0)$. This simplifies $O\left((P_n)_{n=1}^N, P^y\right)$ to 
  \begin{align*}
  &O\left((P_n)_{n=1}^N, P^y\right) \\
  & = \E_{Y \sim P^y} \sum_{n=1}^N \sum_{b=1}^B w_b^n(t) [\tau_n(Y-q_n) + \max(q_n-Y,0) - \tau_n(Y-f_b^n(t)) - \max(f_b^n(t)-q_n,0)] \\
  & = \E_{Y \sim P^y} \sum_{n=1}^N \sum_{b=1}^B w_b^n(t) [\max(q_n-Y,0) - \max(f_b^n(t)-q_n,0)] + C, 
  \end{align*}
  where $C = -\sum_{n=1}^N (V_n^* - V_{n-1}^*) \tau_n = - \sum_{n=1}^{N-1}(\tau_{n+1}-\tau_n) V_n^*$ does not depend on $P^y$. Now we divide cases into $y < q_1$, $q_1 \leq y < q_2$, ..., $q_{N-1} \leq y < q_N$: 
  \begin{align*}
  & \sum_{n=1}^N \sum_{b=1}^B w_b^n(t) [\max(q_n-y,0) - \max(f_b^n(t)-q_n,0)] \\
  & = 
  \begin{cases}
      - \sum_{n=1}^N \sum_{b=1}^B w_b^n \max(y - f_b^n(t), 0) & y < q_1 \\
      - \bigl( \sum_{s=n+1}^N \sum_{b=1}^B w_b^s \max(y - f_b^s(t), 0) \\
      \qquad\quad + \sum_{s=1}^n \sum_{b=1}^B w_b^s \max(f_b^s(t)-y, 0) \bigr) + V_n^* & y \in [q_n, q_{n+1}), \, n \in [N-1]  \\
      - \sum_{n=1}^N \sum_{b=1}^B w_b^n \max(f_b^n(t) - y, 0) & y \geq q_N .
  \end{cases}
  \end{align*}
  In the interval $(-\infty, q_1)$, the quantity is maximized at $y = \theta_0^* = f_\text{min}(t)$, and its maximum value is $0$. In interval $[q_n, q_{n+1})$, the quantity is maximized at $y = \theta_n^*$ (contained in the interval as shown in Proposition \ref{prop:theta_n_pn_sorted_wql}), and its maximum value is $0$. In interval $(q_{N-1}, \infty)$, the quantity is maximized at $y = \theta_N^* = f_\text{max}(t)$, and its maximum value is $0$. Since $P^{y*}$ only puts mass on the points $\{\theta_n^*\}_{n=0}^N$, $O\left((P_n^*)_{n=1}^N, P^{y*}\right) = \max_{P^y} O\left((P_n^*)_{n=1}^N, P^y\right) = C = - \sum_{n=1}^{N-1}(\tau_{n+1}-\tau_n) V_n^* \leq 0$.
  
  Next, $O\left((P_n^*)_{n=1}^N, P^{y*}\right) = \min_{P_n \in \cp(\R), n \in [N]} O\left((P_n)_{n=1}^N, P^{y*}\right)$ follows directly from the propriety of $\rho_\tau$ and the fact that $q_n$ is $\tau_n$-level quantile of $P^{y*}$. 
\end{proof}

Putting all these together, the omniprediction error bound for Algorithm \ref{alg:omni-twoplayer_online_wql} is the following:
\begin{align}
& \sup_{\ell \in \cl_{wql},\, b \in [B]} \frac{1}{T} \sum_{t=1}^T [\ell(\mathbf{q}^G(t), Y_t) - \ell(f_b(t), Y_t)] \nonumber\\
&\quad \leq \, \eta \left(\frac{1}{T} \sum_{t=1}^T \max_{n,b} \left(\rho_{n}(q^G_n(t), Y_t) - \rho_{n}(f_b^n(t), Y_t)\right)^2\right) + \frac{\log BN}{\eta T}. \label{eq:omni_twoplayer_wql_error_bound}
\end{align}

\begin{proof}
  \begin{align*}
    & \sup_{\ell \in \cl_{wql},\, b \in [B]} \frac{1}{T} \sum_{t=1}^T [\ell(\mathbf{q}^G(t), Y_t) - \ell(f_b(t), Y_t)] \\
    & = \max_{n \in [N], b \in [B]} \frac{1}{T} \sum_{t=1}^T [\rho_n(q^G_n(t), Y_t) - \rho_n(f_b^n(t), Y_t)] \\
    & \leq \frac{1}{T} \sum_{t=1}^T \sum_{n=1}^N \sum_{b=1}^B w_b^n(t) [\rho_n(q^G_n(t), Y_t) - \rho_n(f_b^n(t), Y_t)] \\
    & \quad + \frac{\eta}{T} \sum_{t=1}^T \sum_{n=1}^N \sum_{b=1}^B w_b^n(t) \left(\rho_{n}(q^G_n(t), Y_t) - \rho_{n}(f_b^n(t), Y_t)\right)^2 + \frac{\log BN}{\eta T} \\
    & \underset{\text{Thm. } \ref{thm:minimax_wql}}{\leq} \eta \left(\frac{1}{T} \sum_{t=1}^T \max_{n,b} \left(\rho_{n}(q^G_n(t), Y_t) - \rho_{n}(f_b^n(t), Y_t)\right)^2\right) + \frac{\log BN}{\eta T}.
  \end{align*}
\end{proof}

\subsection{Comparison between two algorithms} \label{subsec:comparison_wql}
We conclude the section by comparing the two algorithms in terms of time complexity, quantile crossing, and omniprediction error bound.

\paragraph{Time complexity.}
The time complexity of Algorithm \ref{alg:wql_hedge} is $O(NB)$ per time step. Algorithm \ref{alg:omni-twoplayer_online_wql} is more complicated, where the main computational bottleneck lies in computing $V_n^*$ and $\theta_n^*$ for each $n \in [N-1]$ (see Equation~\eqref{eq:Vn_def_wql}). A naive computation of the minimizers $\{\theta_n^*\}$ requires $O(N^3 B^2)$ per each time step. Appendix~\ref{appsubsec:wql_fast} introduces a subroutine that computes all $\{\theta_n^*\}$ and $\{V_n^*\}$ in $O(K\log K + NB)$ time, where $K$ is the number of unique values in $\{f_b^n(t)\}_{b,n}$, by exploiting the structure of $(V_n^*)_{n=0}^{N}$.

\paragraph{Quantile crossing.}
Due to Proposition \ref{prop:theta_n_pn_sorted_wql}, quantile forecasts from Algorithm \ref{alg:omni-twoplayer_online_wql} do not exhibit quantile crossing. On the other hand, Algorithm \ref{alg:wql_hedge} does not guarantee such a property, and post-sorting the resulting quantile does not always decrease the loss. Appendix~\ref{appsec:post_sorting} provides examples where post-sorting rather increases the loss.

\paragraph{Omniprediction error bound.}
Comparing the omniprediction error bound \eqref{eq:hedge_wql_error_bound} and \eqref{eq:omni_twoplayer_wql_error_bound}, we notice that the former bound is a function of the error of base forecasters, $\rho_n(f_b^n(t), Y_t)$, while the latter bound is a function of relative error of two-player game based algorithm output with respect to base forecasters, $\rho_n(q^G_n(t), Y_t) - \rho_n(f_b^n(t), Y_t)$. This suggests that the two-player based algorithm would achieve a better omniprediction error for harder prediction tasks since the error of base forecasters is likely to be larger compared to the relative error.

\section{Experiment details and additional plots}

\subsection{Experiment details} \label{appsubsec:experiment_details}
During the pandemic, the United States Center for Disease Control and Prevention solicited forecasts of future disease outcomes from research groups throughout the country \cite{cramer2022covidhub}. In our empirical illustration of our method, we focus in particular on forecasts of the future number of people that will be hospitalized with Covid-19 on a given day. Below is the detailed information about the data: 
\begin{itemize}
  \item Horizon: 1, 2, 3, 4 weeks ahead
  \item Location (total $52$): 51 state-level series (50 states + DC) and the U.S. aggregate series 
  \item Dates: 2020-12-29 to 2023-06-05 (127 weeks)
  \item Base forecasters (total $18$): CU-select, GT-DeepCOVID, COVIDhub-baseline, Karlen-pypm, JHU\_IDD-CovidSP, MOBS-GLEAM\_COVID, USC-SI\_kJalpha, JHUAPL-Bucky, JHUAPL-Gecko, MUNI-ARIMA, CMU-TimeSeries, UMass-trends\_ensemble, PSI-DICE, IHME-CurveFit, CUB\_PopCouncil-SLSTM, BPagano-RtDriven, LANL-GrowthRate, UVA-Ensemble
  \item Ensemble models (total $2$): COVIDhub-4\_week\_ensemble (COVIDhub-Ens), JHUAPL-SLPHospEns (JHUAPL-SLPEns)
  \item Quantile levels (total $23$): [0.01, 0.025, 0.05, 0.1, 0.15, 0.2, 0.25, 0.3, 0.35, 0.4, 0.45, 0.5, 0.55, 0.6, 0.65, 0.7, 0.75, 0.8, 0.85, 0.9, 0.95, 0.975, 0.99] 
\end{itemize}

We choose the grid spacing for each region so that the resulting discretization yields sufficiently fine resolution across time series based on their respective ranges. In particular, this choice ensures more than $500$ grid points for all $208 = 52 \times 4$ time series, and more than $1{,}000$ grid points for over $150$ time series. The corresponding grid spacing used for each region is as follows:
\begin{itemize}
  \item RI, ND, AK, VT, WY: 2
  \item NM, ID, NE, WV, HI, NH, ME, DE, SD: 5
  \item OR, UT, NV, IA, KS, MS, MT, DC: 10
  \item VA, WA, AZ, TN, MA, IN, MO, MD, CO, WI, MN, SC, AL, LA, KY, OK, CT, AR: 20
  \item PA, IL, OH, GA, NC, MI, NJ: 50
  \item CA, TX, FL, NY: 100
  \item US: 500
\end{itemize}

\subsection{Additional plots} \label{appsubsec:add_plots}

\begin{figure}[htbp]
  \centering
  \begin{subfigure}{0.98\textwidth}
      \centering
      \includegraphics[width=\textwidth]{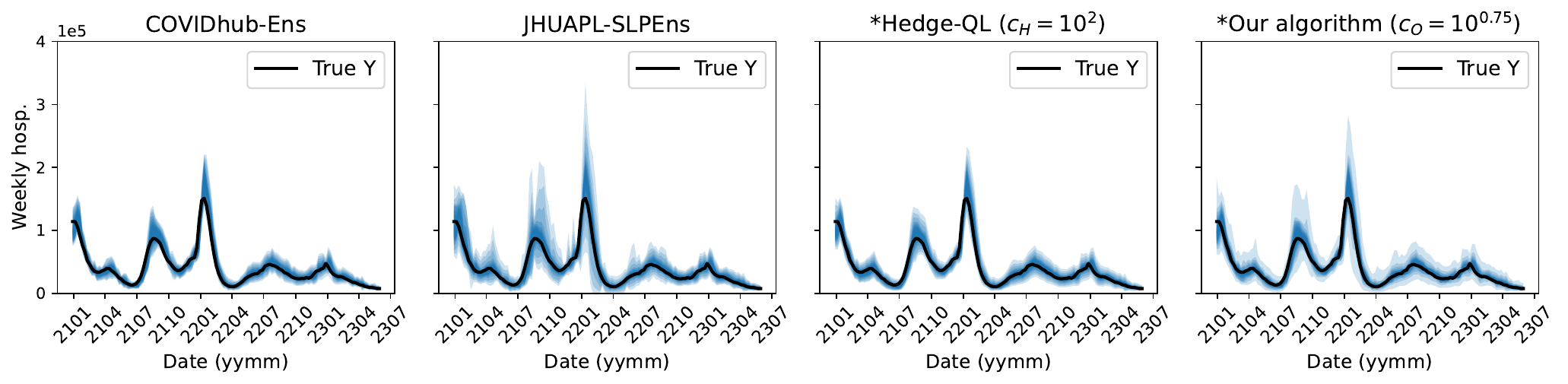}
      \caption{$1$-week-ahead prediction.}
      \label{fig:quantile_estimates_us_1wk}
  \end{subfigure}

  \vspace{0.5em}

  \begin{subfigure}{0.98\textwidth}
      \centering
      \includegraphics[width=\textwidth]{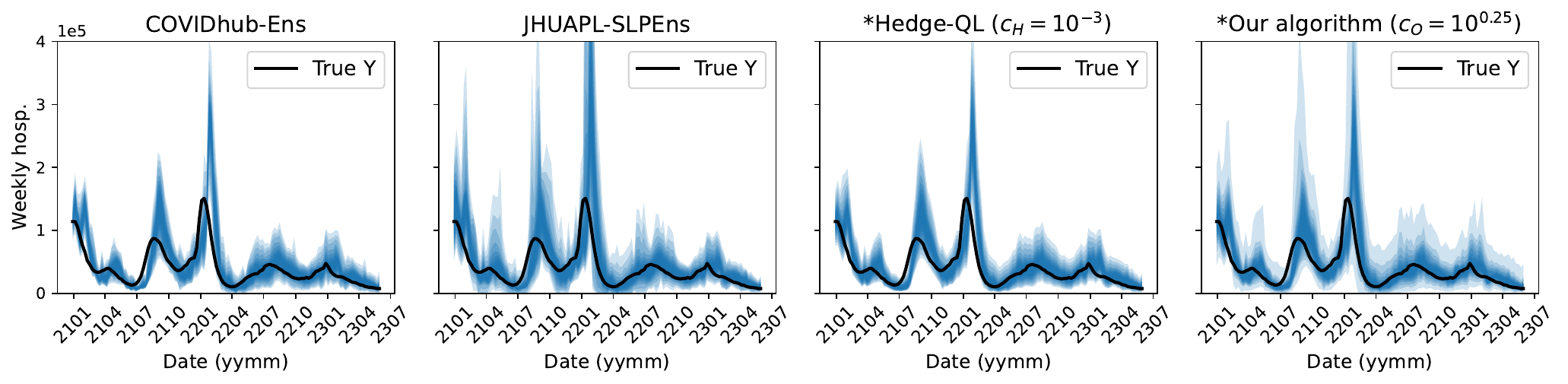}
      \caption{$4$-week-ahead prediction.}
      \label{fig:quantile_estimates_us_4wk}
  \end{subfigure}

  \caption{
  Quantile forecasts for US-agg weekly COVID-19 hospitalizations at $1$-week-ahead and $4$-week-ahead horizons. The learning rate for \texttt{*Hedge-QL} and \texttt{*Our algorithm} is selected to achieve the best total omni-error. The black curve shows the observed hospitalizations. The blue shaded regions show nested quantile prediction bands: darker regions correspond to more central quantile intervals, while lighter regions correspond to wider intervals.
  }
  \label{fig:quantile_estimates_us}
\end{figure}

Figure~\ref{fig:quantile_estimates_us} compares quantile forecasts for US-agg hospitalizations. During the pandemic, COVIDhub-Ens was constructed as a weighted average of the quantile estimates of the base forecasters, whereas JHUAPL-SLPEns was obtained by first constructing predictive distributions and then averaging the resulting CDFs \cite{jhuapl_slphospens}. 
Hedge-QL also aggregates quantile estimates via weighted averaging. 
It is well-known that averaging quantile forecasts produces sharper intervals than averaging predictive distributions~\cite{lichtendahl2013qorcdf}. Consistent with this, COVIDhub-Ens and Hedge-QL produce relatively sharp prediction bands, whereas JHUAPL-SLPEns yields wider intervals. 
Interestingly, although our algorithm only uses quantile forecasts, it produces intervals that are slightly wider than those of COVIDhub-Ens and Hedge-QL, yet noticeably sharper than JHUAPL-SLPEns, while still closely tracking the observed trajectory.

Figure~\ref{fig:HQL_sensitivity}, 
analogous to Figure~\ref{fig:sensitivity}, 
shows how the performance of Hedge-QL varies with the learning rate. The figure shows the realized total omni-error and total quantile loss realized by Hedge-QL for $1$- to $4$-week-ahead horizons on the US-agg data with learning rate constants $c_H \in \{10^{i/4} \mid i=-4,\ldots,12\}$. In addition, the bottom row shows that the quantile crossing of quantile estimates of Hedge-QL could be substantial, especially for large $c_H$s. Our algorithm, by construction, never produces crossings.

\begin{figure}[htbp]
  \centering
  \begin{subfigure}{0.98\textwidth}
      \centering
      \includegraphics[width=\textwidth]{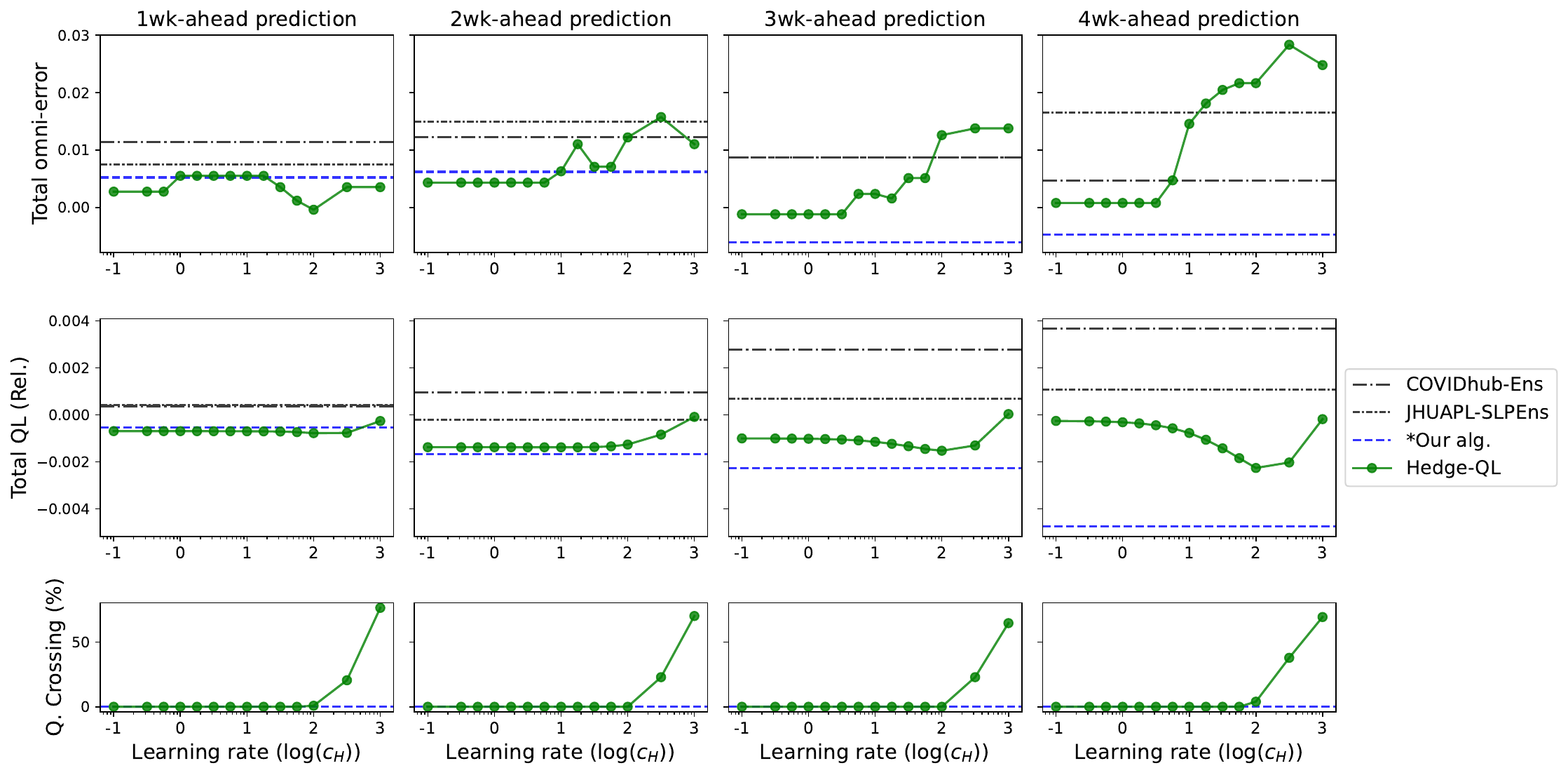}
  \end{subfigure}

  \caption{
    Sensitivity of Hedge-QL to the choice of learning rate. 
    Each panel shows the performance of Hedge-QL as a function of the learning rate (green curve) as measured by the total omni-error (top row) or total quantile loss relative to best base forecaster (middle row) on the US-agg data. Panels from left to right shows results at one to four week time horizons, while horizontal lines indicate the performance of \texttt{*Our algorithm} and the two baseline ensembles. Lower values indicate better performance. The bottom row shows the percentage of time steps at which the quantile estimates of Hedge-QL exhibit quantile crossing.
  }
  \label{fig:HQL_sensitivity}
\end{figure}

Figure~\ref{fig:grid_sensitivity} shows the effect of the grid resolution, which determines the trade-off between the computational cost of our algorithm (Proposition~\ref{prop:omni_subroutine}) and the granularity of its predictions. Across all four horizons, the total quantile loss of our algorithm with a fixed learning rate constant $c_O$ varies little with the grid unit. This indicates that a coarse grid, and hence a small $m$, suffices in practice: the accuracy of our method is largely insensitive to this discretization choice.

\begin{figure}[htbp]
  \centering
  \includegraphics[width=1.0\textwidth]{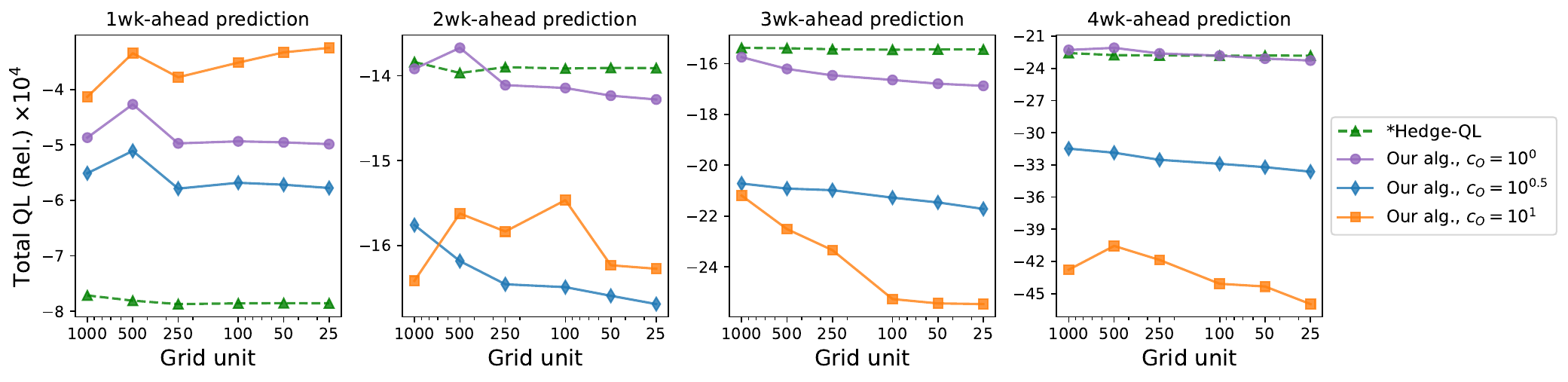}
  \caption{
    Sensitivity of our algorithm to the grid resolution.
    Each panel shows the total quantile loss relative to the best base forecaster on the US-agg data as a function of the grid unit $\{1000, 500, 250, 100, 50, 25\}$ (a smaller unit yields a larger grid size $m$). Lower values indicate better performance.
  }
  \label{fig:grid_sensitivity}
\end{figure}

\begin{figure}[htbp]
  \centering
  \includegraphics[width=0.8\textwidth]{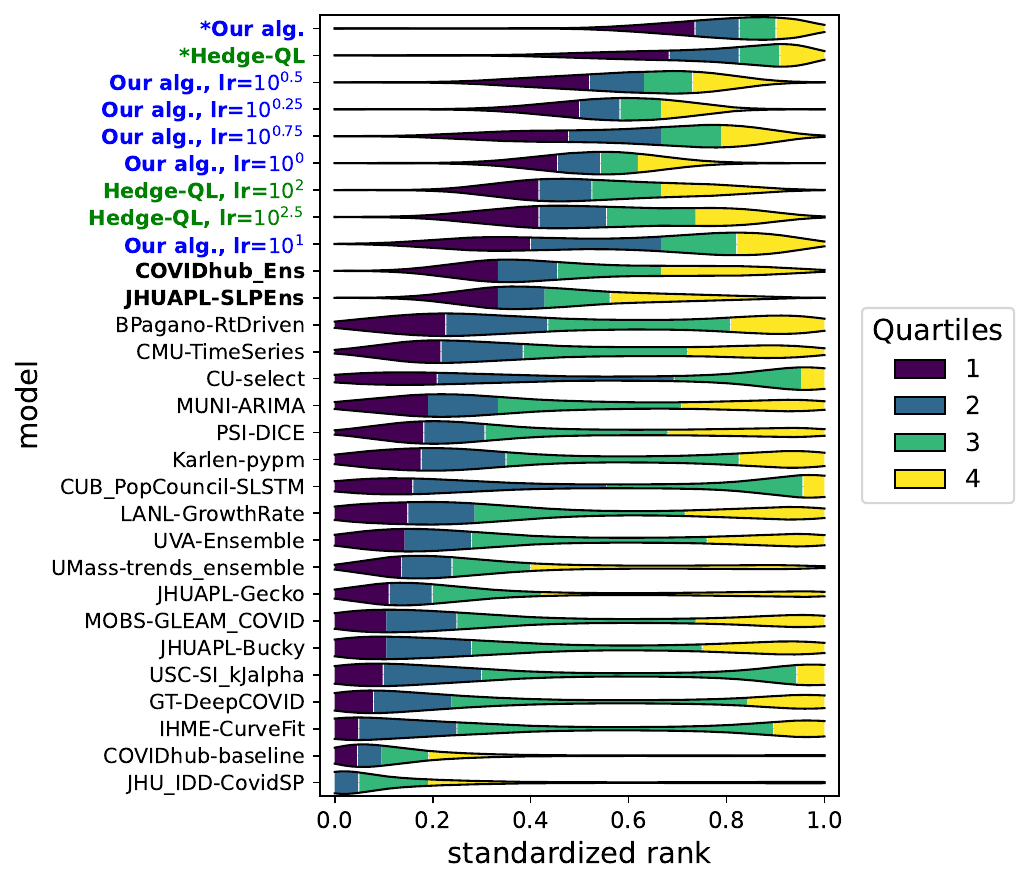}
  \caption{
    Distribution of standardized ranks of instantaneous quantile loss across all (horizon, location, time) tuples for the $51$ state-level series. 
    For each tuple, we rank the available forecasts by quantile loss and normalize the ranks to $[0,1]$. 
    The colored segments represent the empirical distribution of ranks grouped into quartiles. 
    }
  \label{fig:rank_distribution}
\end{figure}

Lastly, Figure~\ref{fig:rank_distribution} summarizes performance of all algorithms, including $18$ base forecasters using standardized ranks. Models in the column of the figure are ordered by their first-quartile performance (lowest quartile of standardized rank), with higher placement indicating better performance.
Besides \texttt{*Our algorithm} and \texttt{*Hedge-QL}, multiple instances of our algorithm with a fixed learning rate are ranked higher than COVIDHUB-Ens and JHUAPL-SLPEns as well as all of the base forecasters. 
In particular, the both the first and third quartile of standardized ranks of all instances  of our algorithm are comparable to or higher than the two ensemble models. This implies that our method, for a wide range of learning rates, can be viewed as better and more robust model compared to existing ensemble models, COVIDHUB-Ens and JHUAPL-SLPEns.

\section{The effect of post-sorting on proper loss functions} \label{appsec:post_sorting}

Prior work \cite{fakoor2023flexible} showed that post-sorting never increases the loss for \textbf{unweighted} sums of pinball losses at different quantile levels. Unfortunately, the same does not hold for \textbf{weighted} sums of pinball losses. We present two results in this section. First, for a class of weighted sums of pinball losses at two quantile levels, we fully characterize the conditions under which post-sorting always decreases the loss. Conversely, when these conditions fail, there exists an instance $(Y, x_1, x_2)$ for which post-sorting increases the loss. Second, we present a subclass of loss functions where post-sorting does not increase the loss.

Define the pinball loss at quantile level $\tau$ as $\rho_\tau(x) := \tau x\indi{x \geq 0} +(1-\tau)|x|\indi{x < 0} = (\tau - \indi{x < 0})x$.
\begin{proposition}
For $(\alpha, \beta) \in (0, 1)^2$ with $\alpha < \beta$, define a weighted quantile loss as
\[
\ell(x_1, x_2) := \E_Y [w_\alpha \rho_\alpha (Y-x_1) + w_\beta \rho_\beta (Y-x_2)].
\]
If $(w_\alpha, w_\beta)$ satisfies one of the following conditions, then post-sorting does not increase the loss:
\begin{enumerate}
  \item $w_\alpha = w_\beta$
  \item $w_\alpha < w_\beta$ and $(1-\alpha)w_\alpha \geq (1-\beta)w_\beta$
  \item $w_\alpha > w_\beta$ and $\beta w_\beta \geq \alpha w_\alpha$.
\end{enumerate}
Conversely, if $(w_\alpha, w_\beta)$ does not satisfy any of the above conditions, then there exists $(Y, x_1, x_2)$ such that post-sorting increases the loss.
\end{proposition}

\begin{proof}
Choose any $x_1 > x_2$. The loss before post-sorting is $\ell(x_1, x_2)$ and the loss after post-sorting is $\ell(x_2, x_1)$. 
\begin{align*}
& \ell(x_1, x_2) - \ell(x_2, x_1) \\
& = w_\alpha \E_Y \left[(\alpha - \indi{Y < x_1})(Y-x_1) - (\alpha - \indi{Y < x_2})(Y-x_2) \right] + \\ 
& \qquad w_\beta \E_Y \left[(\beta - \indi{Y < x_2})(Y-x_2) - (\beta - \indi{Y < x_1})(Y-x_1) \right] \\
& = w_\alpha E_Y \left[\left\{-\alpha(x_1 - x_2) + \indi{Y < x_1}x_1 - \indi{Y < x_2}x_2 - \indi{x_2 \leq Y < x_1}Y \right\}\right] + \\
& \qquad w_\beta E_Y \left[ \left\{\beta(x_1 - x_2) + \indi{Y < x_2}x_2 - \indi{Y < x_1}x_1 + \indi{x_2 \leq Y < x_1}Y \right\} \right] \\
& = \E_Y \left[\left\{-\alpha(x_1 - x_2) + \indi{x_2 \leq Y < x_1}(x_1 - Y) + \indi{Y < x_2}(x_1 - x_2)\right\} \right] w_\alpha + \\
& \qquad \E_Y \left[\left\{\beta(x_1 - x_2) - \indi{x_2 \leq Y < x_1}(x_1 - Y) - \indi{Y < x_2}(x_1 - x_2)\right\} \right] w_\beta 
\end{align*} 
Denote the coefficient of $w_\alpha, w_\beta$ as $c_\alpha, c_\beta$ respectively. Notice that $c_\alpha + c_\beta = (\beta-\alpha)(x_1 - x_2) > 0$. Therefore $\ell(x_1, x_2) - \ell(x_2, x_1) > 0$ if $w_\alpha = w_\beta$ (Condition I).

Now, we consider Condition II: $w_\alpha < w_\beta$ and $(1-\alpha)w_\alpha \geq (1-\beta)w_\beta$. Denote $A = \P(Y < x_2)$, $B = \P(x_2 \leq Y < x_1)$ and observe that $\E_Y \indi{x_2 \leq Y < x_1}(x_1 - Y) \in [0, B(x_1 - x_2)]$. Then the difference in loss is written as:
\begin{align*}
& \ell(x_1, x_2) - \ell(x_2, x_1) \\
& = \left(w_\alpha(A-\alpha) + w_\beta(\beta-A)\right)(x_1 - x_2) - (w_\beta - w_\alpha)\E_Y \indi{x_2 \leq Y < x_1}(x_1 - Y) \\
& \geq -(w_\beta - w_\alpha)A + \beta w_\beta - \alpha w_\alpha - (w_\beta - w_\alpha)B(x_1 - x_2) \\
& \geq -(w_\beta - w_\alpha)A + (w_\beta - w_\alpha) - (w_\beta - w_\alpha)B(x_1 - x_2) \\
& \geq (w_\beta - w_\alpha)(1-A-B)(x_1 - x_2) \geq 0.
\end{align*}
Note that the second inequality used $\beta w_\beta - \alpha w_\alpha \geq w_\beta - w_\alpha$ which follows from Condition II and $A+B = \P(Y < x_1) \leq 1$. Condition III can be proved in a similar manner. \smallskip

Next, we prove that there exists $(Y, x_1, x_2)$ such that post-sorting increases the loss if $(w_\alpha, w_\beta)$ does not satisfy any of three conditions. First consider the case $w_\alpha < w_\beta$ and $(1-\alpha)w_\alpha < (1-\beta)w_\beta$. Let $Y \equiv x_2$ almost surely. Then $A = 0$, $B = 1$, and $\E_Y \indi{x_2 \leq Y < x_1}(x_1 - Y) = x_1 - x_2$. Therefore,
\begin{align*}
& \ell(x_1, x_2) - \ell(x_2, x_1) \\
& = \left(w_\alpha(A-\alpha) + w_\beta(\beta-A)\right)(x_1 - x_2) - (w_\beta - w_\alpha)\E_Y \indi{x_2 \leq Y < x_1}(x_1 - Y) \\
& = \left(w_\alpha(0-\alpha) + w_\beta(\beta-0)\right)(x_1 - x_2) - (w_\beta - w_\alpha)(x_1 - x_2) \\
& = (w_\beta - w_\alpha)(x_1 - x_2) > 0.
\end{align*}
Therefore, post-sorting increases the loss. Next, consider the case where $w_\alpha > w_\beta$ and $\beta w_\beta < \alpha w_\alpha$. Let $Y \equiv x_1$ almost surely. Similar calculation shows that post-sorting increases the loss.
\end{proof}

We prove another subclass of proper loss functions where post-sorting does not increase the loss. From Theorem \ref{thm:decomp_single_q} and Theorem \ref{thm:decomp_multi_q}, a proper loss for multiple quantile levels can be written as the following form under mild regularity conditions:
\begin{align} \label{eq:multi_q_consistent_loss}
\ell((q_n)_{n=1}^N) = \E_Y \sum_{n=1}^{N}(\indi{Y< q_n} - \tau_n)(g_n(q_n) - g_n(Y)),  
\end{align}
where $g_n$ is a non-decreasing left-continuous function, for $n \in [N]$.

\begin{proposition}
If a loss function $\ell$ defined as \eqref{eq:multi_q_consistent_loss} satisfies $g_n(x) = g(x)$ for all $i \in [N]$, then post-sorting does not increase the loss.
\end{proposition}
\begin{proof}
Due to additivity of the loss function, it suffices to show $\ell(x_1, x_2, y) - \ell(x_2, x_1, y) \geq 0$ for $x_1 > x_2$ where $x_n$ corresponds to the $\tau_n$-level quantile estimate.
\begin{align*}
& \ell(x_1, x_2) - \ell(x_2, x_1) \\
& = \E_Y [ (\indi{Y < x_1} - \tau_1)(g_1(x_1) - g_1(Y)) - (\indi{Y < x_2} - \tau_1)(g_1(x_2) - g_1(Y)) + \\
& \quad \qquad (\indi{Y < x_2} - \tau_2)(g_2(x_2) - g_2(Y)) - (\indi{Y < x_1} - \tau_2)(g_2(x_1) - g_2(Y))] \\
& = \E_Y [\tau_1(g_1(x_2)-g_1(x_1)) + \indi{Y < x_1}g_1(x_1) - \indi{Y < x_2}g_1(x_2) - \indi{x_2 \leq Y < x_1}g_1(Y) + \\
& \quad \qquad \tau_2(g_2(x_1)-g_2(x_2)) + \indi{Y < x_2}g_2(x_2) - \indi{Y < x_1}g_2(x_1) + \indi{x_2 \leq Y < x_1}g_2(Y)]
\end{align*}
Now using the fact that $g_1 \equiv g_2 \equiv g$, most of the terms cancel out and we have 
\begin{align*}
\ell(x_1, x_2) - \ell(x_2, x_1) = (\tau_2 - \tau_1)(g(x_1) - g(x_2)) \geq 0.
\end{align*}
\end{proof}

\section{Additional details}

\subsection{Regularity conditions} \label{appssec:regularity}
This subsection provides details of the regularity conditions in Theorem \ref{thm:decomp_multi_q}. In Proposition 4.2 in \cite{fissler2016multiq}, the conditions are stated in terms of the identification function $V$ and its expectation over $Y$, $\bar{V}$ (see Definition 2.2 in \cite{fissler2016multiq}). For the $\tau$-level quantile functional, the identification function on $q,y \in [0,1]$ and $F$ a cdf supported on $[0,1]$ is as follows:
\[
V(q,y) = \one(y \leq x) - \tau 
\qquad \text{and} \qquad 
\bar{V}(q,F) := \E_{Y \sim F} V(q,Y) = F(q) - \tau. 
\]
For multiple quantile levels, $\{\tau_n\}_{n=1}^N$, the identification function is as follows: 
\begin{align*}
V((q_n)_{n=1}^N,y) 
&= \left(\one(y \leq q_1) - \tau_1, \,\, \one(y \leq q_2) - \tau_2, \,\, \ldots, \one(y \leq q_N) - \tau_N \right), \\
\bar{V}((q_n)_{n=1}^N, F) 
&= \left(F(q_1) - \tau_1, \,\, F(q_2) - \tau_2, \,\, \ldots, F(q_N) - \tau_N \right). 
\end{align*}

We now state the regularity conditions for Theorem \ref{thm:decomp_multi_q}. The conditions are referred to as (V1), (V3), (V4), (VS1), (S2), and (F1) in \cite{fissler2016multiq}. Among these, it is straightforward to check that (V1), (V3), (V4), and (F1) hold for quantile functionals. We state the remaining conditions (S2) and (VS1), which are conditions on the loss function, below.

\paragraph{Assumption (VS1)} The complement of the set 
\[
C := \{\left((x_n)_{n=1}^N, y\right) \in \R^N \times \R \mid \ell\left((x_n)_{n=1}^N,\cdot\right) \text{ are continuous at } y\}
\]
has $(N+1)$-dimensional Lebesgue measure zero. 

\paragraph{Assumption (S2)} For every continuously differentiable cumulative distribution function $F$ supported on $[0,1]$, the function $\bar{\ell}(\cdot, F) := \E_{Y \sim F}\ell(\cdot, Y)$ is continuously differentiable, and its gradient is locally Lipschitz continuous. Furthermore, $\bar{\ell}(\cdot, F)$ is twice continuously differentiable at $(x_n)_{n=1}^N = (q_{\tau_n}(F))_{n=1}^N$.

\subsection{Omniprediction for multi-horizon forecast} \label{appsec:multi-horizon}
We note that our omniprediction framework can be extended to multi-horizon prediction task, where the ground truth $Y_t$ and base forecasters' predictions $f_b^n(t)$ are not scalars but vectors. By denoting such vector ground truth and vector predictions as $Y_t = (Y_{t,1}, ..., Y_{t,H})$ and $f_b^n(t) = (f_b^n(t,1), ..., f_b^n(t,H))$, respectively, we can define the omniprediction objective for multiple-horizon prediction as
\[
\sup_{\ell \in \cl_1,\, b \in [B], h \in [H]} \frac{1}{T} \sum_{t=1}^T \E_{q_{n,h} \sim P_{n,h}(t)} [\ell\left((q_{n,h})_{n=1}^N, Y_{t,h}\right)] - \E[\ell\left((f_b^n(t,h))_{n=1}^N, Y_{t,h}\right)],
\]
which means we want to minimize the maximum of regret for each horizon rather than the average of regret over all horizons, which is commonly used in practice. Then, with Algorithm \ref{alg:omni_online} in hand, the task reduces into solving the following minimax program:
\begin{align*}
    \min_{\substack{P_{n,h} \in \Delta_m, \\ \forall n,h}} \max_{\substack{P^y_{h} \in \Delta_m, \\ \forall h}} \sum_{h=1}^H  \sum_{n=1}^N \sum_{b=1}^B  w_{ibh}^n \E_{q_{n,h} \sim P_{n,h}, Y_h \sim P^y_{h}}[\ell_{\tau_n, \theta_i}(q_{n,h}, Y_h) - \ell_{\tau_n,\theta_i}(f_b^n(t,h), Y_h)].
\end{align*}
However, this minimax program is of form 
\[
    \min_{\substack{P_{n,h} \in \Delta_m, \\ \forall n,h}} \max_{\substack{P^y_{h} \in \Delta_m, \\ \forall h}} \sum_{h=1}^H g_h\left((P_{n,h})_{n=1}^N, P^y_{h}\right) 
    = \sum_{h=1}^H \min_{\substack{P_{n,h} \in \Delta_m, \\ \forall n}} \max_{P^y_{h} \in \Delta_m} g_h\left((P_{n,h})_{n=1}^N, P^y_{h}\right) ,
\]
and this reduces into solving $H$ minimax program of form \eqref{eq:Pn_def_minimax}, solution of which is provided in Theorem \ref{thm:omni_guarantee}, separately for each horizon. Therefore, running Algorithm \ref{alg:omni_online} for each horizon separately, and combining the solutions together, will solve the omniprediction problem for multiple-horizon prediction.

\subsection{Regret bound of Hedge algorithm}
We provide the well-known regret bound of the Hedge algorithm, which is used frequently applied in the proofs of this paper.

\begin{theorem}[Regret of Hedge (e.g., Theorem 1.5 of \cite{hazanoco})] \label{appthm:hedge_regret}
  Consider an online learning problem with $B$ experts receiving bounded losses 
  $\{\ell_{b}(t)\}_{b \in [B], t \in [T]}$.
  
  Suppose that at time step $t$ we make the same prediction as expert $i$ with probability
  \[
  w_b(t) = \frac{\exp\!\left(-\eta \sum_{s < t} \ell_b(s)\right)}
  {\sum_{b'=1}^m \exp\!\left(-\eta \sum_{s < t} \ell_{b'}(s)\right)},
  \]
  for some $\eta > 0$. Then,
  \[
  \sum_{t=1}^T \sum_{b=1}^B w_b(t)\ell_b(t)
  \;\le\;
  \min_{1 \le b \le B} \sum_{t=1}^T \ell_b(t)
  \;+\; \eta \sum_{t=1}^T \sum_{b=1}^B w_b(t)\ell_b^2(t)
  \;+\; \frac{\log B}{\eta}.
  \]
  \end{theorem}


\end{document}